\documentclass{article}

\usepackage{dilab_arxiv}
\usepackage{amsmath,amsfonts,bm}
\usepackage{amsthm}

\def\eqref#1{equation~\ref{#1}}
\def\1{\bm{1}}

\DeclareMathAlphabet{\mathsfit}{\encodingdefault}{\sfdefault}{m}{sl}
\SetMathAlphabet{\mathsfit}{bold}{\encodingdefault}{\sfdefault}{bx}{n}

\usepackage[capitalize,noabbrev]{cleveref}

\usepackage{hyperref}
\usepackage{url}
\newtheorem{theorem}{Theorem}[section]
\newtheorem{lemma}[theorem]{Lemma}

\usepackage{amsmath,amsfonts,amssymb}
\usepackage{mathtools}
\usepackage{algorithm}
\usepackage{algpseudocode}
\usepackage{enumerate}
\usepackage{array}

\title{Lipschitz Bandits with Arbitrary Feedback Delays}
\runningtitle{Lipschitz Bandits with Arbitrary Feedback Delays}
\date{arXiv preprint, \today}

\paperlogo{\includegraphics[height=1.5cm]{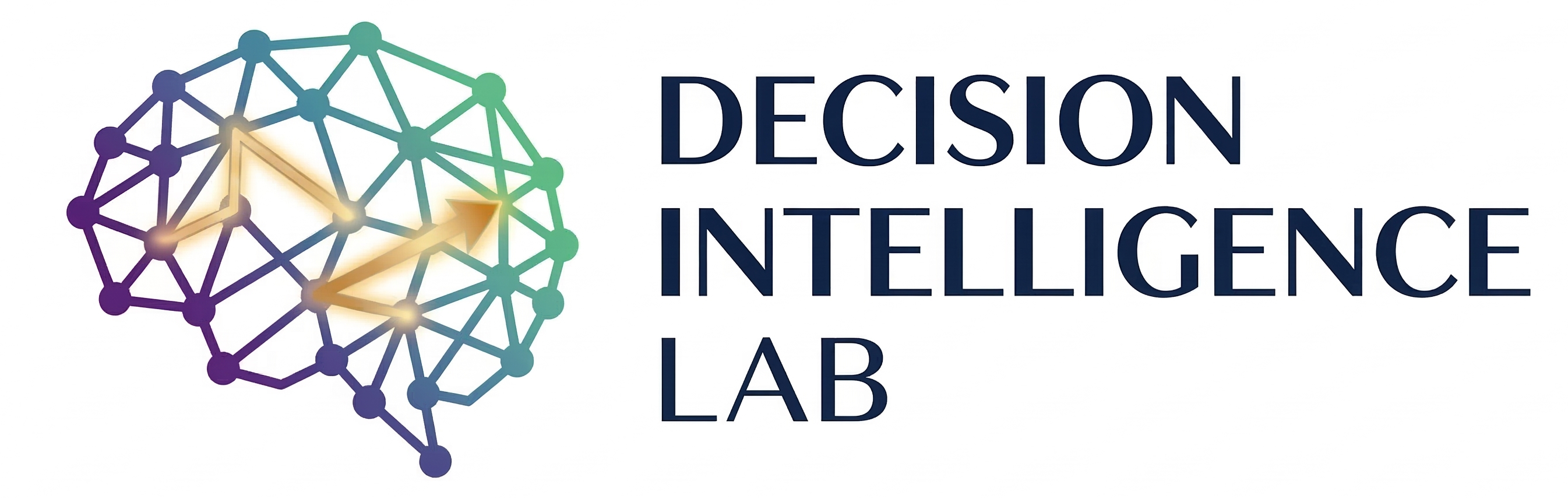}}

\author{
  Yuhao Liu$^{1}$, Yu Chen$^{1}$, and Longbo Huang$^{1\,\text{\faEnvelope}}$
  \\[0.3em]\normalfont
  $^1$Institute for Interdisciplinary Information Sciences, Tsinghua University
  \\
  \quad \text{\faEnvelope}\ Correspondence: longbohuang@tsinghua.edu.cn
}

\begin{document}

\maketitle
\thispagestyle{fancy}

\begin{abstract}
The Lipschitz bandit problem extends the traditional multi-armed bandit framework to continuous action spaces by assuming that the reward functions satisfy a Lipschitz condition. This work investigates Lipschitz bandits under arbitrary feedback delays, where reward signals are not received immediately upon taking an action but after an arbitrarily chosen delay. We consider both stochastic and adversarial reward settings, proposing an elimination-based algorithm and an EXP3-based algorithm, respectively. For both settings, our algorithms achieve a regret bound of $\tilde{O}\left(T^{\frac{d_z+1}{d_z+2}}+\sqrt{D}\right)$ over a time horizon $T$ with total delay $D$, where the main difference between settings lies in the definition of the zooming dimension $d_z$. Our bounds match existing delay-free regret guarantees for Lipschitz bandits and characterize the additional $\tilde{O}(\sqrt{D})$ impact introduced by feedback delays.
\end{abstract}

\section{Introduction}

The Lipschitz bandit framework \citep{continuum95,kleinberg2008multi} extends the classical multi-armed bandit (MAB) formulation to continuous action spaces. In this setting, the agent selects actions from a metric space and receives stochastic rewards whose expectations vary smoothly according to a Lipschitz-continuous function. A standard approach for handling continuous action domains is discretization, which converts the continuous domain into a finite-armed bandit problem \citep{kleinberg2019bandits,feng2022lipschitz}. This approach requires balancing a fundamental trade-off: a finer discretization reduces the bias of the continuous approximation but increases the model complexity and sample regret of the finite-armed problem.

In classical formulations, reward feedback is assumed to be immediate. However, in many real-world applications, feedback is delayed, forcing the agent to continue making decisions without access to the most recent observations. While bandit problems with delayed feedback have been extensively studied, their extension to Lipschitz bandits remains largely unexplored. Recently, \citet{liu2026lipschitz} studied Lipschitz bandits with delayed feedback, but their analysis is restricted to stochastic delays. Therefore, in this work, we address this gap by studying Lipschitz bandits under more general delay models, where delays can be arbitrarily chosen before the procedure starts.

\subsection{Lipschitz Bandits with Delayed Feedback}

Let $(X,\delta)$ be a compact metric space, where $\delta: X \times X \to \mathbb{R}_{\ge 0}$ denotes the metric on $X$. Let $T$ be the number of rounds. For each $t$, let $\mu_t: X \to [0,1]$ be the reward function, which is assumed to be Lipschitz continuous, i.e., $|\mu_t(x) - \mu_t(y)| \le \delta(x,y)$ for all $x,y \in X$. We call the tuple $(X,\delta,\{\mu_t\}_{t=1}^T)$ an instance of the Lipschitz bandit problem. In the standard stochastic reward settings, the reward function remains the same, that is, $\mu_t\equiv\mu$ for all $t=1,\dots,T$. We also consider the "oblivious" adversarial reward settings, where the reward function $\mu_t$ is chosen arbitrarily before all interaction starts.

The interaction proceeds as follows. At each round $t=1,\dots,T$, the agent selects an arm $x_t \in X$, and a reward $r_t$ is generated after the arm selection. The reward $r_t$ is drawn from a distribution that depends only on $x_t$ and satisfies $\mathbb{E}[r_t] = \mu_t(x_t)$.  The agent's goal is to minimize the cumulative regret, which is defined as $$R(T)=\max_{x^*\in X}\sum_{t=1}^T\mu_t(x^*)-\sum_{t=1}^T\mu_t(x_t).$$

In order to optimize the cumulative regret, the agent should use previous feedback to update its policy. However, in the delayed feedback setting, when a reward is generated, it may not be immediately observed by the agent. Instead, each reward $r_t$ is observed after a delay of $d_t$ rounds. Therefore, at the beginning of round $t$, the agent observes the history of actions $\{x_i : i < t\}$ and the set of rewards $\{r_i : i + d_i < t\}$. A delay $d_t = 0$ indicates an immediate feedback. Without loss of generality, we assume that $t+d_t\le T$, since $d_t=T-t$ already ensures that the reward of round $t$ be never used to update the agent policy. In this work, we assume that $\{d_t\}_{t=1}^T$ is \textbf{arbitrarily} determined before the interaction starts (i.e., oblivious delays), but the agent only knows $d_t$ after the $(t+d_t)$-th round.

Another important concept for a Lipschitz bandit problem instance is the zooming dimension, whose definition is based on the covering number. A ball in $(X,\delta)$ with diameter $l$ can be written as $B(x,l):=\{y\in X:\delta(x,y)<l/2\}$. Let $S\subseteq X$, we write $\text{cov}(S,l)$ as the minimum number of balls of diameter $l$ to cover the set $S$. It is known that we can define the covering dimension of $X$ using the covering number: $$d:=\inf\{n\ge0:\exists c>0,\forall \epsilon>0,\text{cov}(X,\epsilon)\le c\epsilon^{-n}\}.$$
For the definition of zooming dimension, one need to first define some sub-optimal region $X_\epsilon\subseteq X$, and the $c$-zooming dimension for a problem instance as typically defined as $$d_z=\inf\{n\ge0:\forall r\in(0,1],\text{cov}(X_\epsilon,\epsilon)\le c\epsilon^{-n}\}.$$ The explicit definition of zooming dimension is given in Sections \ref{sec:result-stochastic} and \ref{sec:result-adversarial} for the corresponding settings.

Compared to the covering dimension, the zooming dimension does not control the scale of whole space. Instead, it only focuses on a smaller region that is nearly optimal. Intuitively, the zooming dimension measures the difficulty of distinguishing the optimal arm from the other arms. Note that the zooming dimension depends on the reward function $\mu$, and might be significantly smaller than the dimension $d$ of the whole space $X$.

In the absense of feedback delays, the state-of-the-art algorithms for Lipschitz bandits with stochastic rewards achieve a regret bound of $\tilde{O}\left(T^{\frac{d_z+1}{d_z+2}}\right)$, where $d_z$ denotes the zooming dimension \citep{kleinberg2019bandits}. Moreover, similar regret bound are also achieved for adversarial rewards, where the only difference lies in the definition of zommming dimension. In this work, we extend existing results by providing algorithms with feedback delays and achieving regret bounds that explicitly depends on both $T$ and the delay budget $D$.

\paragraph{Main Contribution.}

We summarize our contribution as follows. We also compare our results with previous results in Table \ref{tab:comparison}.

\begin{itemize}
    \item We investigate the performance of an existing elimination-based algorithm in stochastic Lipschitz bandits, within the setting of arbitrary feedback delays. By developing a greedy-based analysis that characterizes the impact of worst-case delay patterns, we show in Theorem \ref{thm:oblivious} that the algorithm achieves an regret bound of $\tilde{O}\left(T^{\frac{d_z+1}{d_z+2}}+\sqrt{D}\right)$.

    \item We design an EXP3-based algorithm for adversarial rewards with arbitrary feedback delays. By carefully selecting the zooming schedule and the weight-update procedure, we prove that, when the delay budget $D$ and the zooming dimension $d_z$ are known, the proposed algorithm achieves $\tilde{O}\left(T^{\frac{d_z+1}{d_z+2}}+\sqrt{D}\right)$ regret under a modified notion of zooming dimension.

    \item Our regret bounds recover the optimal rates of the corresponding delay-free settings in both stochastic and adversarial cases. Furthermore, our analysis demonstrates that arbitrary feedback delays contribute an additional penalty of order $\tilde{O}(\sqrt{D})$ to the regret, highlighting the precise impact of delayed feedback on Lipschitz bandit problems.
\end{itemize}

% \begin{table}[h!]
% \label{tab:results}
% \begin{minipage}{\textwidth}
%     \centering
%     \setlength{\extrarowheight}{2pt}
%     \begin{tabular}{>{\raggedright\arraybackslash}m{3.4cm}|>{\raggedright\arraybackslash}m{3.8cm}|>{\raggedright\arraybackslash}m{4.7cm}}
%          \hline
%          Scheme & Stochastic Rewards & Adversarial Rewards \\
%          \hline
%          Optimality Gap $\Delta$ & $\Delta(x)=\mu(x^*)-\mu(x)$ & $\Delta_t(x)=\frac{1}{t}\sum_{s\in\tilde{O}_t}[\mu_t(x^*)-\mu_t(x)]$\\
%          \hline
%          Definition of Suboptimal Region $X_\epsilon$ & $\{x:\Delta(x)\le O(\epsilon)\}$ & $\{x:\exists t\ge\Omega(\epsilon^{-2}),\Delta_t(x)\le O(\epsilon\ln T)\}$\\
%          \hline
%          Definition of Zooming Dimension $d_z$ & \multicolumn{2}{|l}{$d_z=\inf\{n\ge0:\forall\epsilon\in(0,1],\mathrm{cov}(X_\epsilon,\epsilon)\le c\epsilon^{-n}\}$}\\
%          \hline
%          Regret without Delay & $O\left(T^{\frac{d_z+1}{d_z+2}}\left(c\log T\right)^{\frac{1}{d_z+2}}\right)$ & $O\left(T^{\frac{d_z+1}{d_z+2}}\sqrt{d2^d}c^{\frac{1}{d_z+2}}\log T\right)$\\
%          \hline
%          Penalty due to Delay & $O\left(\sqrt{D\log D}\right)$ & $O\left(\sqrt{dD\log T}\right)$ \\
%          \hline
%     \end{tabular}
%     \caption{Summarization of our main results.}
% \end{minipage}
% \end{table}

\begin{table}[h!]
\begin{minipage}{\textwidth}
    \centering
    \setlength{\extrarowheight}{2pt}
    \begin{tabular}{>{\raggedright\arraybackslash}m{3.2cm}|c|c|c}
         \hline
         & Reward & Delay & Regret\\
         \hline
         \citet{kleinberg2008multi}, Theorem 2.4 & Stochastic & No Delay & $O\left(T^{\frac{d_z+1}{d_z+2}}(c\log T)^{\frac{1}{d_z+2}}\right)$ \footnote{$d_z$ represents the zooming dimension with multiplier $c$.}\\
         \hline
         \citet{podimata2021adaptive}, Theorem 2 & Adversarial & No Delay & $O\left(T^{\frac{d_z+1}{d_z+2}}d^{1/2}(c\cdot C_{\text{dbl}})^{\frac{1}{d_z+2}}\log^5T\right)$ \footnote{The definition of $d_z$ is different for adversarial bandits.}\\
         \hline
         \citet{liu2026lipschitz}, Theorem 1 & Stochastic & Bounded & $O\left(T^{\frac{d_z+1}{d_z+2}}(c\log T)^{\frac{1}{d_z+2}}+c\tau_{\max}\left(\frac{T}{\log T}\right)^\frac{d_z}{d_z+2}\right)$\\
         \hline
         \citet{liu2026lipschitz}, Theorem 3 & Stochastic & Stochastic & $O\left(\frac{1}{p}T^{\frac{d_z+1}{d_z+2}}(c\log T)^{\frac{1}{d_z+2}}+Q(p)\right)$ \footnote{$Q(\cdot)$ represents the quantile function.}\\
         \hline
         \textbf{Theorem \ref{thm:oblivious} (Ours)} & Stochastic & {\textbf{Arbitrary}} & $O\left(T^{\frac{d_z+1}{d_z+2}}\left(c\log T\right)^{\frac{1}{d_z+2}}+\sqrt{D\log D}\right)$\\
         \hline
         \textbf{Theorem \ref{thm:adversarial} (Ours)} & Adversarial & {\textbf{Arbitrary}} & $O\left(T^{\frac{d_z+1}{d_z+2}}\sqrt{d2^d}c^{\frac{1}{d_z+2}}\log T+\sqrt{dD\log T}\right)$\\
         \hline
    \end{tabular}
    \caption{Comparison with Previous Results.}
    \label{tab:comparison}
\end{minipage}
\end{table}

\section{Related Works}

\paragraph{Lipschitz Bandits.}

Lipschitz bandits, first introduced by \citet{continuum95}, extend the classical multi-armed bandit (MAB) framework to continuous action spaces. \citet{kleinberg2008multi} further generalized this setting to arbitrary metric spaces. Building on these foundational results, subsequent works \citep{feng2022lipschitz,kleinberg2019bandits} adapt techniques from finite-armed bandits to the Lipschitz setting, achieving provably efficient algorithms with near-optimal regret guarantees.

A variety of extensions to the standard Lipschitz bandit framework have also been studied. \citet{podimata2021adaptive} consider an oblivious adversary in a non-stationary environment, where the reward function evolves over time but is fixed in advance. \citet{kang2023robust} investigate adversarial reward corruption and derive regret bounds that depend explicitly on the corruption budget. Similarly, \citet{nguyen2025nonstationary} analyze non-stationary settings, with regret controlled by the number of significant distributional shifts. Other works explore alternative reward assumptions and problem settings, including heavy-tailed rewards \citep{lu2019optimal}, as well as multi-agent and quantum extensions studied by \citet{chakraborty2026multiagent} and \citet{yi2026quantum}, respectively.

Most recently, \citet{liu2026lipschitz} studied Lipschitz bandits with stochastic delayed feedback, providing algorithms that handle both bounded and unbounded delays. In contrast, the setting we are considering focuses on a more general delay model, where delays may be arbitrary and potentially adversarial. In particular, we aim to derive regret bounds that explicitly depend on the total delay budget.

\paragraph{Bandits with Delayed Feedback.}

Multi-armed bandit problems with delayed observations have been well studied in the literature. \citet{joulani2013online} studies general online learning problems with delayed feedback, and then provides a modified UCB algorithm for the stochastic MAB scenarios. \citet{vernade2017stochastic} studies the stochastic MAB problem with delayed and censored reward, while assuming a known delay distribution. Adversarial bandits with arbitrary feedback delays are also widely studied. For example, \citet{thune2019nonstochastic,bistritz2019onlineexp3} provide typical results, where the regret is added by $O(\sqrt{D\log K})$ compared to the case without delay. \citet{zimmert20a} refines the bound so that it is tighter in some cases. In addition, there are a variety of works about topics such as best-of-both-worlds algorithms \citet{bobw2022} and linear bandits \citet{vernade20a}.
\section{Results for Stochastic Rewards}
\label{sec:result-stochastic}
% \subsection{Oblivious Adversary}

We first consider the stochastic reward setting with arbitrary oblivious feedback delays. In this section, we state the algorithm and establish a regret upper bound for the stochastic reward setting. We consider arbitrary delays that are fixed before the first round starts, but remain unknown to the agent until the corresponding feedback is received. Following \citet{liu2026lipschitz}, we assume that the reward is generated as $r_t=\mu(x_t)+\epsilon_t$, where $\{\epsilon_t\}_{t=1}^T$ are independent sub-Gaussian random variables satisfying $\Pr\{|\epsilon_t|>\xi\}\le 2e^{-\xi^2/2}$. For simplicity of presentation, we specialize our discussion to the metric space $X=[0,1]^d$ equipped with the metric $\delta(x,y)=|x-y|_{\infty}$. Under this metric, each ball in $(X,\delta)$ is a $d$-dimensional cube, and the covering dimension of $(X,\delta)$ is exactly $d$. This setting is widely adopted in the literature \citep{feng2022lipschitz,podimata2021adaptive} and can be generalized to arbitrary metric spaces by assuming access to a covering oracle (see \citealt{liu2026lipschitz} and Section 6 of \citealt{podimata2021adaptive}).

Before introducing the algorithm, we first define the \textbf{zooming tree} over $X=[0,1]^d$. Each node $u$ in the zooming tree corresponds to an $\ell_\infty$-ball (i.e., a cube) in $X$, and we denote its diameter by $L(u)$. The root node $u_{\mathrm{root}}$ corresponds to the entire space $X$. Each node $u$ is recursively partitioned into $2^d$ children, each having diameter $L(u)/2$, such that the children form an exact partition of $u$. Therefore, for any node $u$ at depth $h(u)$, its diameter satisfies $L(u)=2^{-h(u)}$. We denote by $C(u)$ the set of children of node $u$.

We directly adopt the Delayed Lipschitz Phased Pruning algorithm proposed by \citet{liu2026lipschitz}. The algorithm proceeds in phases $m=1,2,\ldots$. In each phase, it maintains an "active" set $U_m$ consisting of nodes at depth $m-1$ in the zooming tree. For every node $u\in U_m$, the algorithm selects an arbitrary point $x\in u$ and plays it. Whenever a delayed reward signal $(u,r)$ is received, the algorithm updates the empirical mean reward $\hat{\mu}(u)$ and the number of observed rewards $v(u)$. Each node is sampled repeatedly until its number of observed rewards reaches the required threshold $v_m$. After all nodes in $U_m$ have received sufficiently many observations, the algorithm prunes nodes with $\Omega(2^{-m})$ estimated optimality gap, zooming-in the remaining nodes by adding their children to the active set, and proceeds to the next phase. The pseudo-code is provided in Algorithm \ref{alg:prune-oblivious}.

\begin{algorithm}[htbp]
\caption{Delayed Lipschitz Phased Pruning \citep{liu2026lipschitz}}
\label{alg:prune-oblivious}
\begin{algorithmic}[1]

\Require $(X=[0,1]^d,\delta=\ell_\infty)$, $T$, and parameter $\delta>0$.

\State Initialize timer $t=1$, phase counter $m=1$, and $U_m=\{u_{\text{root}}\}$.

\While{$t\le T$}
    \State $U_m^+\gets U_m$; set $v(u)\gets0$, $\hat{\mu}(u)\gets0$ for all $u\in U_m$. \qquad {\color{red}// Initialize phase $m$.}
    \While{$t\le T$ and $U_m^+\neq\emptyset$}
        \For{$u\in U_m^+$}
            \State Sample $x\in u$ and pull the arm $x$.
            \For{any incoming reward $(u',r)$ with $u'\in U_m^+$}
                \State $\hat{\mu}(u')\gets(\hat{\mu}(u')\cdot v(u')+y)/(v(u')+1)$. \qquad {\color{red}// Update the average reward.}
                \State $v(u')\gets v(u')+1$.
                \State Remove $u'$ from $U_m^+$ if $v(u')\ge v_m=(2\ln T+\ln(2/\delta))\cdot 2^{2m-1}$.
            \EndFor
            \State $t\gets t+1$. Break if $t>T$.
        \EndFor
    \EndWhile
    \State $\hat{\mu}_m^*\gets\max_{u\in U_m}\hat{\mu}(u)$.
    \For {$u\in U_m$}
        \State Remove $u$ it from $U_m$ if $\hat{\mu}_m^*-\hat{\mu}(u)>8\cdot 2^{-m}$. \quad{\color{red}// Prune the nodes with high instantaneous regret.}
    \EndFor
    \State $U_{m+1}\gets\bigcup_{u\in U_m}C(u)$. \qquad {\color{red}// Zoom-in the remaining nodes.}
    \State $m\gets m+1$.
\EndWhile
\end{algorithmic}
\end{algorithm}

The key components of Algorithm \ref{alg:prune-oblivious} are the pruning step (line 16) and the zoom-in step (line 17). The pruning step guarantees that the remaining nodes contain only near-optimal actions, with their optimality gap bounded by $O(2^{-m})$ in phase $m$. The zoom-in step then refines the discretization of the action space by partitioning each remaining node into smaller balls with half the diameter. In this case, the algorithm automatically adapts the discretization level according to the observed rewards, thereby balancing the trade-off between exploration and approximation error.

Moreover, Algorithm \ref{alg:prune-oblivious} does not require knowledge of the total delay budget $D$. As we will show below, the algorithm achieves a regret bound that automatically adapts to the actual total delay $D$.

\paragraph{Regret Bound.} For an arm $x\in X$, define its optimality gap as $\Delta(x)=\max_{x^*\in X}\mu(x^*)-\mu(x)$. Define the $\epsilon$-optimal region as $$X_\epsilon=\{x\in X:\Delta(x)\le\epsilon\}.$$
The $c$-zooming dimension of the problem instance $(X,\delta,\mu)$ is defined as $$d_z=\inf\{n\ge0:\forall\epsilon\in(0,1],\mathrm{cov}(X_\epsilon,\epsilon/16)\le c\epsilon^{-n}\},$$
where $\mathrm{cov}(X,\epsilon)$ denotes the minimum number of zooming-tree nodes with diameter at most $\epsilon$ required to cover $X$. The constant $1/16$ and the restriction to tree-node coverings are introduced only for notational convenience and can be absorbed into the constant $c$. As shown later, the dependence of our regret bound on $c$ is mild.

Our main theoretical result is summarized in the following theorem. We show that Algorithm \ref{alg:prune-oblivious} achieves nearly the same regret rate as Lipschitz bandits without delays, with an additional penalty that depends on the total feedback delay.

\begin{theorem}
    With probability at least $1-\delta$, the regret of Algorithm \ref{alg:prune-oblivious} satisfies 
    $$R(T)\lesssim T^{\frac{d_z+1}{d_z+2}}\left(c\log\frac{T}{\delta}\right)^{\frac{1}{d_z+2}}+\sqrt{D\log D},$$
    where $D=\sum_{t=1}^{T}d_t$, and $d_z$ denotes the $c$-zooming dimension of the problem instance $(X,\delta,\mu)$.
    \label{thm:oblivious}
\end{theorem}

Theorem \ref{thm:oblivious} provides a high-probability bound for the achieved regret. If we take $\delta=1/\mathrm{poly}(T)$, the $\frac{1}{\delta}$ can be omitted in the representation. The regret bound in Theorem \ref{thm:oblivious} consists of two components. The first term is of order $\tilde{O}(T^{\frac{d_z+1}{d_z+2}})$, which matches the optimal regret rate for Lipschitz bandits without delays established by \citet{kleinberg2019bandits}. The second term, $O(\sqrt{D\log D})$, captures the additional regret caused by arbitrary feedback delays.

Although Algorithm \ref{alg:prune-oblivious} is identical to the one proposed by \citet{liu2026lipschitz}, our analysis establishes robustness against arbitrary delay sequences, whereas previous analyses only considered stochastic delays. Moreover, for comparison, in finite-armed bandits, the additional regret caused by delays is typically bounded by $O(\sqrt{D\log K})$, where $K$ denotes the number of arms (see, e.g., \citealt{thune2019nonstochastic,zimmert20a}). Our result recovers this dependence in the Lipschitz setting, since the number of active arms generated by the discretization procedure can be as large as $T$, yielding $\log K=O(\log T)$, while the maximum possible total delay satisfies $D=O(T^2)$.

\section{Results for Adversarial Rewards}
\label{sec:result-adversarial}

This section presents our algorithm and the corresponding regret bound for adversarial Lipschitz bandits with delayed feedback. In this setting, both the reward functions $\{\mu_t\}_{t=1}^T$ and the delays $\{d_t\}_{t=1}^T$ are fixed arbitrarily before the first round starts. For analytical convenience, we consider deterministic rewards, i.e., $r_t=\mu_t(x_t)$. As in Section \ref{sec:result-stochastic}, we specialize the metric space to $([0,1]^d,\ell_\infty)$ and adopt the same definition of the zooming tree.

Our algorithm is built upon the EXP3-based algorithm proposed by \citet{podimata2021adaptive}, which extends the classical EXP3 framework \citep{auer2002nonstochastic} to Lipschitz bandits through adaptive discretization. Throughout the interaction, the algorithm maintains a set of nodes $A_t$ in the zooming tree, referred to as the \textit{active nodes} at round $t$. Following the notation of \citet{podimata2021adaptive}, for any node $v$, we define $\mathrm{act}_t(v)=u$ if $u$ is the active ancestor of $v$ at round $t$.

We modify the EXP3-based zooming procedure to accommodate delayed feedback by carefully scheduling the weight updates and zooming operations, where the pseudo-code is provided in Algorithm \ref{alg:adversarial}. The algorithm takes $\eta$, $\beta$, and $\gamma_t$ as parameters, whose values will be specified later. Each round consists of three major steps: action selection (lines 4--7), weight update (lines 8--13), and zooming (lines 14--24).

\begin{algorithm}[htb]
\caption{Delayed Lipschitz EXP3}
\label{alg:adversarial}
\begin{algorithmic}[1]

\Require $(X=[0,1]^d,\delta=\ell_\infty)$, $T$, and parameter $\eta,\beta,(\gamma_t)_{t=1}^T$.

\State Initialize counter $\tau=0$ and active nodes $A_1=\{u_{\text{root}}\}$.
\State Set $w_{0}(u_{\text{root}})\gets1$, $\textnormal{conf}_0(u_{\textbf{root}})\gets\frac{1}{\beta}$.

\For{$t=1,2,\dots,T$}
    \State $p_t(u)\gets\frac{w_\tau(u)}{\sum_{v\in A_t}w_\tau(v)}$ for all $u\in A_t$ \qquad  {\color{red}// Compute distribution $p_t$ on $A_t$.}
    \State $\pi_t(u)\gets(1-\gamma_t)p_{\tau}(u)+\gamma_t/|A_t|$ for all $u\in A_t$.\qquad {\color{red}// Add uniform exploration on nodes.}
    \State Sample $u_t\in A_t$ according to probability $\pi_t$.
    \State Sample $x_t\in u_t$ uniformly at random, and pull arm $x_t$.
    \For{any incoming reward $(u_s,\mu_s(x_s))$}
        \For{$v\in A_t$}
            \State $w_{\tau+1}(v)\gets w_{\tau}(v)\cdot\exp\left(\eta\left(\frac{\mu_s(x_s)1\{\textbf{act}_s(v)=u_s\}+(1+4\ln T)\beta}{\pi_{s}(\textbf{act}_s(v))}\right)\right)$. \quad {\color{red}// Update weights using EXP3 rule.}
        \EndFor
        \State $\tau\gets\tau+1$.
    \EndFor
    \State $A_{t+1}\gets A_t$.
    \For{$v\in A_t$}
        \State $\textnormal{conf}_{t}(v)\gets\textnormal{conf}_{t-1}(v)+\frac{\beta}{\pi_t(v)}$. \qquad {\color{red}// Calculate the total "confidence".}
        \If{$\textnormal{conf}_{t}(v)<t\cdot L(v)$}
            \State $A_{t+1}\gets (A_{t+1}\setminus\{v\})\cup C(v)$. \qquad {\color{red}// Zoom-in the node.}
            \For{$v'\in C(v)$}
                \State $w_\tau(v')\gets w_\tau(v)/|C(v)|$. \qquad {\color{red}// Distribute the weight to all children.}
                \State $\textnormal{conf}_{t}(v')\gets\textnormal{conf}_{t}(v)$.
            \EndFor
        \EndIf
    \EndFor
\EndFor
\end{algorithmic}
\end{algorithm}

In the sampling step, the algorithm selects an active node $u_t$ according to a distribution proportional to its weight and then samples an arbitrary arm $x_t\in u_t$ to play. Following \citet{podimata2021adaptive}, for every active node $u\in A_t$, we construct the estimated reward $$\hat{\mu}_t(u):=\frac{\mu(x_t)1\{u=u_t\}+(1+4\ln T)\beta}{\pi_t(u)}.$$
This estimator consists of an inverse propensity score (IPS) estimator together with an additional confidence correction term. In the absence of delays, $\hat{\mu}_t(u)$ can be immediately used to update the weight of node $u$. However, under delayed feedback, the reward might arrives after $u$ being zoomed-in. If node $u$ remains active at that time, we directly update its weight. Otherwise, the update is transferred to each active descendant $v$ of $u$ at round $t+d_t$.

The remaining component of the algorithm is the zooming step. For each active node $u\in A_t$, we define its cumulative confidence as $\text{conf}_t(u)=\frac{1}{\beta}+\sum_{s=1}^t\frac{\beta}{\pi_s(\textbf{act}_s(u))}$. Intuitively, $\mathrm{conf}_t(u)$ quantifies the accumulated estimation uncertainty associated with node $u$, and the algorithm refines the discretization whenever this uncertainty becomes smaller than the scale of the node. Furthermore, when $\beta<1/2$, the sampling rule guarantees that $\frac{\beta}{\pi_t(u)}<L(u)$ for any node $u$ that is zoomed-in at round $t$. Unlike the weight update step, the zooming decision does not require waiting for delayed feedback and can therefore be performed immediately.

\textbf{Remark:} Define $\tilde{O}_t:=\{s:s+d_s\le t\}$. Then, by the construction of the zooming tree and the delayed update rule, after round $t$, we have the following useful representation of the weights:
\begin{equation}
    w_\tau(u)=L(u)^d\exp\left(\eta\sum_{s\in\tilde{O}_t}\hat{\mu}_s(\textbf{act}_s(u))\right).
    \label{eq:weight}
\end{equation}

\paragraph{Regret Bound.} We first define the adversarial gap of an arm $x\in X$ at round $t$ as $$\Delta_t(x):=\frac{1}{t}\sup_{x^*\in X}\sum_{s\in\tilde{O}_t}[\mu_t(x^*)-\mu_t(x)].$$
That is, we consider only the rounds whose feedback has arrived by round $t$, select the best fixed comparator arm $x^*$ over these observable rounds, and measure the average cumulative optimality gap of arm $x$. In the absence of delays, this definition reduces to the adversarial gap used in \citet{podimata2021adaptive}.

Using this notion of gap, we define the $\epsilon$-optimal region as $$X_\epsilon:=\{x\in X:\exists t\ge\epsilon^{-2}/9,\Delta_t(x)\le O(\epsilon\ln T)\},$$
where the constant hidden in the $O(\cdot)$ notation is specified in Lemma \ref{lem:advgap-single}. In other words, an arm $x$ is considered $\epsilon$-optimal if its adversarial gap is sufficiently small at some sufficiently large time $t$. Based on this definition, we introduce the adversarial $c$-zooming dimension with delayed feedback:
$$d_z:=\inf\{n\ge0:\forall\epsilon\in(0,1],\mathrm{cov}(X_\epsilon,\epsilon)\le c\epsilon^{-n}\},$$
where $\mathrm{cov}(X_\epsilon,\epsilon)$ denotes the minimum number of zooming-tree nodes with diameter at most $\epsilon$ required to cover $X_\epsilon$. As before, restricting the covering sets to nodes in the zooming tree is only for simplicity of presentation and only affects the constant $c$. The worst-case adversarial zooming dimension is $d_z\le d$.

We now present the regret guarantee for Algorithm \ref{alg:adversarial}.

\begin{theorem}
    Assume that the agent has prior knowledge of the total delay budget $D=\sum_{t=1}^{T}d_t$ and the adversarial $c$-zooming dimension $d_z$. If we choose $$\eta=\beta=\Theta\left(\frac{\sqrt{d\ln T}}{\sqrt{2^d}T^{(d_z+1)/(d_z+2)}c^{1/(d_z+2)}\sqrt{\ln T}+\sqrt{D}}\right),\quad\gamma_t=(2+4\ln T)\beta|A_t|,$$
    then Algorithm \ref{alg:adversarial} achieves
    $$\mathbb{E}\left[R(T)\right]\lesssim\sqrt{d2^d}T^{\frac{d_z+1}{d_z+2}}c^{\frac{1}{d_z+2}}\log T+\sqrt{dD\log T}.$$
    \label{thm:adversarial}
\end{theorem}

Theorem \ref{thm:adversarial} shows that Algorithm \ref{alg:adversarial} achieves a regret bound of order $\tilde{O}\left(T^{\frac{d_z+1}{d_z+2}}+\sqrt{D}\right)$ for adversarial Lipschitz bandits with delayed feedback. This rate has the same structure as the stochastic-delay result in Theorem \ref{thm:oblivious}; the main difference lies in the definition of the zooming dimension. In both cases, the zooming dimension is upper bounded by the ambient dimension $d$ in the worst case.

When there are no delays, the first term $\tilde{O}\left(T^{\frac{d_z+1}{d_z+2}}\right)$ recovers the regret guarantee of \citet{podimata2021adaptive}. Moreover, the additional delay penalty $\tilde{O}(\sqrt{D})$ matches the dependence on the total delay budget established for traditional finite-armed bandits with delayed feedback. Unlike the stochastic setting, however, the combination of adversarial rewards and delayed feedback introduces additional difficulty, and achieving the optimal regret guarantee in Algorithm \ref{alg:adversarial} requires the agent to know both the zooming dimension $d_z$ and the total delay budget $D$.
\section{Analysis (Outline)}

\subsection{Stochastic Rewards: A Greedy-Based Analysis}

In this section, we present the key steps in the proof of Theorem \ref{thm:oblivious}. Consider a fixed phase $m=1,2,\ldots$. By applying the concentration inequalities established in \citet{liu2026lipschitz}, we obtain that, with high probability, throughout phase $m$, every arm contained in any node $u\in U_m$ has optimality gap bounded by $O(2^{-m})$. Consequently, the cumulative regret can be decomposed according to the phases as $$R(T)\le\sum_{m\ge 1}O(2^{-m})\times(\text{the total number of pulls in phase }m).$$

Therefore, the remaining challenge is to bound the number of pulls performed in each phase. In the absence of feedback delays, every node $u\in U_m$ is sampled exactly $v_m$ times, and hence the total number of pulls in phase $m$ is $|U_m|v_m$. Furthermore, the size of the active set $|U_m|$ can be controlled through the zooming dimension. Following the analysis of \citet{liu2026lipschitz}, this immediately yields the delay-free regret bound $$R(T)\lesssim T^{\frac{d_z+1}{d_z+2}}\left(c\log\frac{T}{\delta}\right)^{\frac{1}{d_z+2}}.$$

However, under arbitrary feedback delays, the algorithm may sample a node for more than $v_m$ times, since the feedback may not be observed. Let $v_m+a(u)$ denote the total number of pulls for node $u$ during phase $m$. Then, the additional regret incurred during phase $m$ can be bounded by $$\sum_{u\in U_m}a(u)\times O(2^{-m}).$$
Therefore, the remaining task is to control the total number of additional pulls $A:=\sum_{u\in U_m}a(u)$ in terms of the total delay budget $D$. To establish such a bound, we consider the following equivalent inverse problem: given that the algorithm suffers $A$ additional pulls, what is the minimum total delay required to generate such an amount of over-sampling? By lower bounding this minimum delay as a function of $A$, we can then derive an upper bound on $A$ in terms of the actual delay budget $D$.

For convenience, let $u_1,u_2,\ldots,u_{|U_m|}$ denote the order in which nodes are processed in the for-loop, and define $a_i:=a(u_i)$. Our analysis is based on the following two observations.

\textbf{Observation I.} For a fixed sequence $\{a_i\}$, the minimum total delay required to generate these additional pulls can be achieved by greedily postponing the feedback of the latest outstanding rewards. More specifically, suppose that node $u_i$ is pulled $v_m+a_i$ times in total. To create exactly $a_i$ additional pulls, the feedback corresponding to the $v_m$-th through $(v_m+a_i-1)$-th pulls should be delayed until the round of the $(v_m+a_i)$-th pull.

Using Observation I and carefully summing the induced delays, the minimum total delay associated with a fixed sequence $\{a_i\}$ can be written as
$$D=\sum_{i=1}^{|U_m|}\sum_{s=1}^{a_{u_i}}s\cdot\left[\sum_{j=i}^{|U_m|}1_{a_{u_j}\ge s-1}+\sum_{j=1}^{i-1}1_{a_{u_j}\ge s}\right].$$

\textbf{Observation II.} The minimum delay is achieved when the values of $a_i$ are sorted in non-decreasing order. (See Appendix \ref{appendix:proof-stochastic} for the proof.)

Using Observation II, without loss of generality, we assume that $a_1\le a_2\le\cdots\le a_{|U_m|}$. Define $w_s=\sum_{i=1}^{|U_m|}1\{a_i\ge s\}$, which represents the number of nodes with at least $s$ additional pulls. Under this ordering, the above expression simplifies to
$D=\sum_{s\ge 1}s\cdot w_s^2$.
Clearly $A\le D$, so we conclude with a Cauchy-Schwarz argument: $$A=\sum_{s=1}^{D}w_s=\sum_{s=1}^{D}\sqrt{s}w_s\cdot\frac{1}{\sqrt{s}}\le\sqrt{\sum_{s=1}^{D}sw_s^2\cdot\sum_{s=1}^D\frac{1}{s}}= O\left(\sqrt{D\log D}\right).$$

\subsection{Adversarial Rewards: Potential Analysis with Delay}

In this section, we outline the main steps in the proof of Theorem \ref{thm:adversarial}. We first introduce several notations that will be used throughout the analysis. Let $\tau$ be the same as in Algorithm \ref{alg:adversarial}. Let $A_\tau$ denote the corresponding active set after processing $\tau$ feedback signals, and define $p_\tau(u)=\frac{w_\tau(u)}{\sum_{v\in A_\tau}w_\tau(u)}$ for $u\in A_\tau$. For each original round $s$, let $\tau_0(s)$ denote the value of $\tau$ immediately before selecting the action $x_s$, and let $\tau_1(s)$ denote the value of $\tau$ immediately before incorporating the feedback $\mu_s(x_s)$.

Our analysis follows the potential-based proof of EXP3 and the approach of \citet{podimata2021adaptive}. We define the potential function $\Phi(\tau)=\ln\left(\sum_{u\in A_\tau}w_\tau(u)\right)$. Since the zooming step preserves the total weight, it does not change the potential function. Therefore, the potential changes only when the algorithm incorporates a received feedback signal.

Suppose that $\mu_s(x_s)$ is the $\tau$-th received feedback signal. By the weight update rule,
$$\begin{aligned}
    \Phi(\tau)-\Phi(\tau-1)&=\ln\left(\sum_{u\in A_\tau}p_{\tau-1}(u)\exp\left(\eta\hat{\mu}_s(\textbf{act}_s(u))\right)\right)\\
    &\le\ln\left(1+\eta\sum_{u\in A_\tau}p_{\tau-1}(u)\hat{\mu}_s(\textbf{act}_s(u))+\eta^2\sum_{u\in A_\tau}p_{\tau-1}(u)\hat{\mu}_s^2(\textbf{act}_s(u))\right)\\
    &\le\eta\sum_{u\in A_\tau}p_{\tau-1}(u)\hat{\mu}_s(\textbf{act}_s(u))+\eta^2\sum_{u\in A_\tau}p_{\tau-1}(u)\hat{\mu}_s^2(\textbf{act}_s(u)).
\end{aligned}$$
Here, the choices of $\eta$, $\beta$, and $\gamma_t$ guarantee that $\eta\hat{\mu}_s(\cdot)\le 1$. Therefore, the second-order expansion follows from the inequality $e^x\leq1+x+x^2$ for $x\leq1$, and the last step uses $\ln(1+x)\leq x$.

Recall that the estimated reward is defined as $\hat{\mu}_s(v)=\frac{\mu_s(x_s)1\{u_s=v\}+(1+4\ln T)\beta}{\pi_s(v)}$. To extract the actual cumulative reward term $\sum_{s=1}^T\mu_s(x_s)$ from the estimated rewards, it is convenient to replace the distribution $p_{\tau_1(s)}$ appearing in the potential analysis by the distribution $p_{\tau_0(s)}$ used when selecting the action. These two distributions differ only because of delayed feedback. Moreover, when $\gamma_s\le 1/2$, the sampling rule gives $p_{\tau_0(s)}(u)\le\pi_s(u)/(1-\gamma_s)\le(1+2\gamma_s)\pi_s(u)$. Together with the condition $\eta\hat{\mu}_s(\cdot)\le1$, this implies
$$\begin{aligned}
    \Phi(\tau)-\Phi(\tau-1)&\le2\eta\langle\max\{0,p_{\tau_1(s)}-p_{\tau_0(s)}\},\hat{\mu}_s\rangle+\sum_{u\in A_s}p_{\tau_0(s)}(u)\left[\eta\hat{\mu}_s(u)+\eta^2\hat{\mu}_s^2(u)\right]\\
    &\le2\eta\langle\max\{0,p_{\tau_1(s)}-p_{\tau_0(s)}\},\hat{\mu}_s\rangle\\
    &\quad+\eta\mu(x_s)+4\gamma_s\eta+O(\eta\beta\log T)\sum_{u\in A_s}\hat{\mu}_s(u).
\end{aligned}$$
The formal definition of the inner product term is given in Lemma \ref{lem:regret-bound}; we omit the intermediate algebraic steps for brevity. On the other hand, a direct lower bound on the final potential implies $$\Phi(T)-\Phi(0)=\ln\left(\sum_{u\in A_T}w_T(u)\right)\ge\ln w_T(u^*)=d\ln L(u^*)+\eta\sum_{s=1}^T\hat{\mu}_s(\textbf{act}_s(u^*))$$
for any node $u^*$. Combining the upper and lower bounds on the potential yields
$$\begin{aligned}
    \sum_{s=1}^T[\hat{\mu}_s(\textbf{act}_s(u^*))-\mu_s(x_s)]&\le\underbrace{-\frac{d\ln L(u^*)}{\eta}+4\sum_{s=1}^T\gamma_s+O(\beta\log T)\sum_{s=1}^T\sum_{u\in A_s}\hat{\mu}_s(u)}_{\text{Regret without delay}}\\
    &\quad+\underbrace{2\eta\sum_{s=1}^T\langle\max\{0,p_{\tau_1(s)}-p_{\tau_0(s)}\},\hat{\mu}_s\rangle}_{\text{Regret caused by delay}}.
\end{aligned}$$
This inequality provides an upper bound on the "estimated regret" of Algorithm \ref{alg:adversarial}. It naturally decomposes the regret into two components: the standard EXP3-type regret term that is independent of delays, and an additional term that captures the effect of delayed feedback. To obtain the final regret bound, we control these two terms separately through the following four steps.

\begin{enumerate}[i)]
    \item We first show that, for any fixed arm $x^*\in u^*$, the difference $\sum_{s=1}^T[\mu_s(x^*)-\hat{\mu}_s(\textbf{act}_s(u^*))]$ is sufficiently small with high probability. This allows us to convert the surrogate regret on the left-hand side of the previous inequality into the actual regret. (See Lemma \ref{lem:concentration-optimal}.)
    \item We then bound the delay-independent terms on the right-hand side using the final number of active nodes $|A_T|$. (See Lemma \ref{lem:regret-bound-highprob}.)
    \item Next, we control the growth of the zooming tree by showing that with high probability the number of activated nodes satisfies $O(2^dT^{\frac{d_z}{d_z+2}}c^\frac{2}{d_z+2}])$. This result follows from the zooming rule, the definition of the adversarial zooming dimension, and a greedy counting argument. (See Lemma \ref{lem:tree-size}.)
    \item Finally, we analyze the delay-dependent term. Taking expectation, we show that the total variation between the sampling distributions before and after receiving delayed feedback satisfies $\Vert\max\{0,p_{\tau_1(s)}-p_{\tau_0(s)}\}\Vert_1\lesssim\eta[\tau_1(s)-\tau_0(s)]$. By summing over all rounds and carefully accounting for the contribution of each delay, we show that the regret caused by delays is bounded by $O(\eta D)$. (See Lemmas \ref{lem:delay-expectation} and \ref{lem:delay-inner-product-expectation}.)
\end{enumerate}

Combining these results, we obtain $$R(T)\le\frac{O(d\log T)}{\eta}+\eta \cdot O(T|A_T|\ln T+D)\le\frac{O(d\log T)}{\eta}+\eta\cdot O\left(T^{\frac{2(d_z+1)}{d_z+2}}\ln T+D\right).$$
Balancing the learning rate $\eta$ then gives the desired regret bound. Note that selecting the optimal value of $\eta$ requires prior knowledge of both the total delay $D$ and the zooming dimension $d_z$.

\section{Conclusion}

This work studies Lipschitz bandits with arbitrary feedback delays, where the delay sequence is fixed arbitrarily before the interaction begins. We provide an elimination-based algorithm and an EXP3-based algorithm for stochastic and adversarial rewards, respectively. Both algorithms recover the state-of-the-art regret rates of their corresponding delay-free settings. Moreover, feedback delays introduce an additional regret penalty of $O(\sqrt{D\log D})$ in the stochastic setting and $O(\sqrt{dD\log T})$ in the adversarial setting, which have nearly the same order with respect to the total delay budget. These results provide further understanding of the impact of delayed feedback on Lipschitz bandit problems.

In our current analysis for adversarial rewards, a fixed learning rate $\eta$ is required to control the effect of delayed feedback, which further requires prior knowledge of total delay $D$ and zooming dimension $d_z$. A possible future direction is to develop adversarial bandit algorithms that automatically achieve the optimal regret rate without prior knowledge of problem parameters. 

% \section*{Impact Statement}
% Describe the broader impacts, limitations, and responsible-use considerations
% that are specific to this work.

% \section*{Acknowledgements}
% Add acknowledgements and funding information here. Remove this section for an
% anonymous draft when appropriate.

\bibliography{iclr2026_conference}

@article{continuum95,
author = {Agrawal, Rajeev},
title = {The Continuum-Armed Bandit Problem},
journal = {SIAM Journal on Control and Optimization},
volume = {33},
number = {6},
pages = {1926-1951},
year = {1995},
}

@inproceedings{kleinberg2008multi,
  title={Multi-armed bandits in metric spaces},
  author={Kleinberg, Robert and Slivkins, Aleksandrs and Upfal, Eli},
  booktitle={Proceedings of the fortieth annual ACM symposium on Theory of computing},
  pages={681--690},
  year={2008}
}

@article{kleinberg2019bandits,
  title={Bandits and experts in metric spaces},
  author={Kleinberg, Robert and Slivkins, Aleksandrs and Upfal, Eli},
  journal={Journal of the ACM (JACM)},
  volume={66},
  number={4},
  pages={1--77},
  year={2019},
  publisher={ACM New York, NY, USA}
}

@article{feng2022lipschitz,
  title={Lipschitz bandits with batched feedback},
  author={Feng, Yasong and Wang, Tianyu and others},
  journal={Advances in Neural Information Processing Systems},
  year={2022}
}

@inproceedings{
liu2026lipschitz,
title={Lipschitz Bandits with Stochastic Delayed Feedback},
author={Zhongxuan Liu and Yue Kang and Thomas Lee},
booktitle={The Fourteenth International Conference on Learning Representations},
year={2026},
}

@inproceedings{podimata2021adaptive,
  title={Adaptive discretization for adversarial lipschitz bandits},
  author={Podimata, Chara and Slivkins, Alex},
  booktitle={Conference on Learning Theory},
  pages={3788--3805},
  year={2021},
  organization={PMLR}
}

@inproceedings{
kang2023robust,
title={Robust Lipschitz Bandits to Adversarial Corruptions},
author={Yue Kang and Cho-Jui Hsieh and Thomas Chun Man Lee},
booktitle={Thirty-seventh Conference on Neural Information Processing Systems},
year={2023},
}

@inproceedings{
nguyen2025nonstationary,
title={Non-Stationary Lipschitz Bandits},
author={Nicolas Nguyen and Solenne Gaucher and Claire Vernade},
booktitle={The Thirty-ninth Annual Conference on Neural Information Processing Systems},
year={2025},
}

@inproceedings{lu2019optimal,
  title={Optimal Algorithms for Lipschitz Bandits with Heavy-tailed Rewards},
  author={Lu, Shiyin and Wang, Guanghui and Hu, Yao and Zhang, Lijun},
  booktitle={International Conference on Machine Learning},
  pages={4154--4163},
  year={2019},
  organization={PMLR}
}

@misc{chakraborty2026multiagent,
      title={Multi-Agent Lipschitz Bandits}, 
      author={Sourav Chakraborty and Amit Kiran Rege and Claire Monteleoni and Lijun Chen},
      year={2026},
      eprint={2602.16965},
      archivePrefix={arXiv},
      primaryClass={cs.LG}
}

@inproceedings{yi2026quantum,
  title={Quantum lipschitz bandits},
  author={Yi, Bongsoo and Kang, Yue and Li, Yao},
  booktitle={Proceedings of the AAAI Conference on Artificial Intelligence},
  volume={40},
  number={33},
  pages={27844--27851},
  year={2026}
}

@inproceedings{joulani2013online,
  title={Online learning under delayed feedback},
  author={Joulani, Pooria and Gyorgy, Andras and Szepesv{\'a}ri, Csaba},
  booktitle={International conference on machine learning},
  year={2013},
}

@inproceedings{bobw2022,
 author = {Masoudian, Saeed and Zimmert, Julian and Seldin, Yevgeny},
 booktitle = {Advances in Neural Information Processing Systems},
 title = {A Best-of-Both-Worlds Algorithm for Bandits with Delayed Feedback},
 year = {2022}
}

@misc{vernade2017stochastic,
      title={Stochastic Bandit Models for Delayed Conversions}, 
      author={Claire Vernade and Olivier Cappé and Vianney Perchet},
      year={2017},
      eprint={1706.09186},
      archivePrefix={arXiv},
      primaryClass={cs.LG},
}

@inproceedings{vernade20a,
  title = 	 {Linear bandits with Stochastic Delayed Feedback},
  author =       {Vernade, Claire and Carpentier, Alexandra and Lattimore, Tor and Zappella, Giovanni and Ermis, Beyza and Br{\"u}ckner, Michael},
  booktitle = 	 {Proceedings of the 37th International Conference on Machine Learning},
  year = 	 {2020},
}

@InProceedings{zimmert20a,
  title = 	 {An Optimal Algorithm for Adversarial Bandits with Arbitrary Delays},
  author =       {Zimmert, Julian and Seldin, Yevgeny},
  booktitle = 	 {Proceedings of the Twenty Third International Conference on Artificial Intelligence and Statistics},
  year = 	 {2020},
}

@inproceedings{thune2019nonstochastic,
 author = {Thune, Tobias Sommer and Cesa-Bianchi, Nicol\`{o} and Seldin, Yevgeny},
 booktitle = {Advances in Neural Information Processing Systems},
 title = {Nonstochastic Multiarmed Bandits with Unrestricted Delays},
 year = {2019}
}

@inproceedings{bistritz2019onlineexp3,
 author = {Bistritz, Ilai and Zhou, Zhengyuan and Chen, Xi and Bambos, Nicholas and Blanchet, Jose},
 booktitle = {Advances in Neural Information Processing Systems},
 editor = {H. Wallach and H. Larochelle and A. Beygelzimer and F. d\textquotesingle Alch\'{e}-Buc and E. Fox and R. Garnett},
 pages = {},
 publisher = {Curran Associates, Inc.},
 title = {Online EXP3 Learning in Adversarial Bandits with Delayed Feedback},
 volume = {32},
 year = {2019}
}

@article{auer2002nonstochastic,
  title={The nonstochastic multiarmed bandit problem},
  author={Auer, Peter and Cesa-Bianchi, Nicolo and Freund, Yoav and Schapire, Robert E},
  journal={SIAM journal on computing},
  volume={32},
  number={1},
  pages={48--77},
  year={2002},
  publisher={SIAM}
}
\bibliographystyle{dilab_ref}

% Remove this block when the paper has no appendix.
\makeappendixtoc
\appendix
\section{Proof of Theorem \ref{thm:oblivious}}
\label{appendix:proof-stochastic}

In this section, we provide proofs of Theorem \ref{thm:oblivious}. 

\subsection{Regret Bound under Clean Events}

This part follows the analysis by \citet{liu2026lipschitz}. First, we define the following clean event.
\begin{equation}
    \label{eq:clean}
    \mathcal{E}:=\left\{|\hat{\mu}(u)-\mu(x)|\le 4\cdot 2^{-m},\forall x\in u\in U_m,\forall 1\le m\le m^*,\right\}.
\end{equation}
Here, $m^*$ denotes the number of completed phases (i.e., entered pruning step), and $\hat{\mu}(u)$ is the average of received rewards in the corresponding phase. Note that when $u$ is removed $U_m^+$, this average is not updated anymore.

Then, if the event $\mathcal{E}$ holds, we claim that the optimal arm $x^*\in X$ is never pruned. We can prove by induction. Suppose it has not been pruned when phase $m$ begins. Let $x^*\in u\in U_m$, $\hat{\mu}(u^*)=\max_{u\in U_m}\hat{\mu}(u)$, and pick any $x\in u^*$. Then $$\hat{\mu}(u^*)-\hat{\mu}(u)\le 4\cdot 2^{-m}+\mu(x)-\mu(x^*)+4\cdot 2^{-m}\le 8\cdot 2^{-m},$$
and $u$ is not pruned according to the pruning rule.

Moreover, under event $\mathcal{E}$, we also claim that the optimality gap $\Delta(x)$ for $x\in u\in U_m$ is at most $32\cdot 2^{-m}$. This claim is trivially holds for $m=1$. For $m>1$, since the optimal arm $x^*$ is not pruned during phase $m-1$, let $x^*\in u^*\in U_{m-1}$. Moreover, if $x\in u\in\mathcal{B}_{m-1}$ is not pruned, Hence, $$\Delta(x)=\mu(x^*)-\mu(x)\le\hat{\mu}(u^*)-\hat{\mu}(u)+8\cdot 2^{1-m}\le 32\cdot 2^{-m}.$$

With these claims, we can directly control the regret as follows:
$$\begin{aligned}
    R(T)&=\sum_{m=1}^{m^*+1}\sum_{u\in U_m}\sum_{i=1}^{n(B)}\Delta(x_{B,i})\\
    &\le16\cdot 2^{-M}\cdot T+32\sum_{m=1}^{M}\sum_{u\in U_m}n(u)\cdot 2^{-m}\\
    &\lesssim 2^{-M}\cdot T+\sum_{m=1}^{M}\left[|U_m|\cdot v_m\cdot 2^{-m}+2^{-m}\sum_{u\in U_m}a_u\right]\\
    &\lesssim 2^{-M}\cdot T+\sum_{m=1}^{M}\left[|U_m|\cdot(\log T+\log(1/\delta))2^m+2^{-m}\sum_{u\in U_m}a_u\right],
\end{aligned}$$
where $1\le M\le m^*$ is determined later, and $a_u$ is the number of "additional plays" of the node $u$, that is, $a_u:=n(u)-v_m$.

Note that for all $x\in u\in U_m$, its optimality gap is bounded by $32\cdot 2^{-m}=16L(u)$. By the definition of zooming dimension, we have $$|U_m|\le c\cdot 2^{d_z(m-1)}.$$
So,
$$\begin{aligned}
    R(T)&\lesssim 2^{-M}\cdot T+\sum_{m=1}^M[\log (T/\delta)]c\cdot 2^{(d_z+1)m}+\sum_{m=1}^M2^{-m}\sum_{u\in U_m}a_B\\
    &\lesssim 2^{-M}\cdot T+c\log(T/\delta)2^{(d_z+1)M}+\sum_{m=1}^M2^{-m}\sum_{u\in U_m}a_B.
\end{aligned}$$
Following \citet{liu2026lipschitz}, by taking $M=\frac{\log\frac{T}{2^dc\log T}}{d_z+2}$, we have
$$R(T)\lesssim T^{\frac{d_z+1}{d_z+2}}\left(c\log\frac{T}{\delta}\right)^{\frac{1}{d_z+2}}+\sum_{m=1}^M2^{-m}\sum_{B\in\mathcal{B}_m}a_B.$$

\subsection{Proof of Theorem \ref{thm:oblivious}}

According to Lemma \ref{lem:clean-prob}, the clean event $\mathcal{E}$ happens with probability at least $1-\delta$. So we only need to control $\sum_{u\in U_m}a_u$, given that $\sum_{t=1}^Td_t\le D$.

Consider an adversary that tries to maximize $\sum_{u\in U_m}a_u$. During each phase $m$, it is equivalent to consider the problem: fix $A=\sum_{u\in U_m}a_m$, and then minimize the total amount of delays needed to achieve it.

Let the nodes in $U_m$ be ordered as $u_1,u_2,\dots,u_{|U_m|}$, which is the order that are played during the $m$-th phase. First, if we fix $a_{u_i}$, the optimal strategy for the adversary is to delay all of the $v_m$-th to $(v_{m}+a_{u_i}-1)$-th observation to the step when we play $u_i$ for the $(v_m+a_{u_i})$-th time. Then, with this strategy, given all $a_{u_i}$, the total delay used can be written as follows: $$D=\sum_{i=1}^{|U_m|}\sum_{s=1}^{a_{u_i}}s\cdot\left[\sum_{j=i}^{|U_m|}1_{a_{u_j}\ge s-1}+\sum_{j=1}^{i-1}1_{a_{u_j}\ge s}\right].$$
To explain this formula, let $t_{i,s}$ be the step we play ball $u_i$ for the $(v_m+s)$-th time, then $\sum_{j=i}^{|U_m|}1_{a_{u_j}\ge s-1}+\sum_{j=1}^{i-1}1_{a_{u_j}\ge s}$ is exactly $t_{i_s}-t_{i,s-1}$. It is then added to the delay for the $s$ former plays in the ball $u_i$, increasing the total delay by $s\cdot(t_{i,s}-t_{i,s-1})$.

Now, we prove a claim: the optimal choice of $\{a_{u_i}\}$ satisfies $a_{u_i}\le a_{u_{i+1}}$ for $i=1,2,\dots,|U_m|-1$. To prove the claim, consider an arbitrary choice of $\{a_{u_i}\}$, if $p=a_{u_i}>a_{u_{i+1}}=q$, by interchanging $a_{u_i}$ and $a_{u_{i+1}}$, the only changed part in the formulation of minimum total delay $D$ is $$\sum_{s=q+1}^ps\cdot1_{a_{u_{i+1}}\ge s-1},$$
which is changed to $$\sum_{s=q+1}^ps\cdot1_{a_{u_{i}}\ge s}.$$
By comparing the indicators, we can observe that $D$ is decreased by $(q+1)$. Therefore the claim holds.

With the above claim, we assume that the sequence $\{a_{u_i}\}$ is non-decreasing. Let $w_s:=\sum_{i=1}^{|U_m|}1_{a_{u_i}\ge s}$, that is, the "width" of the $s$-th layer, then $\sum_{s\ge 1}w_s=\sum_{u\in U_m}a_B=A$, and we can again rewrite $D$ as $$D=\sum_{s\ge 1}s\cdot w_s^2.$$
Clearly $A\le D$, so $w_s=0$ for $s>D$. Using Cauchy-Schwarz inequality, 
$$A=\sum_{s=1}^D\sqrt{s}w_s\frac{1}{\sqrt{s}}\le\sqrt{\sum_{s\ge1}sw_s^2}\cdot\sqrt{\sum_{s=1}^D\frac{1}{s}}\le O(\sqrt{D\log D}).$$

Finally, plug it into the regret formula with $\sum_{m=1}^Mr_m=O(1)$, we obtain $$R(T)\lesssim T^{\frac{d_z+1}{d_z+2}}\left(c\log\frac{T}{\delta}\right)^{\frac{1}{d_z+2}}+\sqrt{D\log D}.$$

\subsection{Auxiliary Lemmas}

\label{appendix:lemmas}

We adapt the following lemma from \citet{liu2026lipschitz}.

\begin{lemma}(Lemma 11 by \citet{liu2026lipschitz})
    Let $m^*$ be the last round that it is complete in Algorithm \ref{alg:prune-oblivious} (i.e., entered the pruning step). Define the following clean event $$\mathcal{E}':=\left\{|\hat{\mu}_m(u)-\mu(x)|\le 2\cdot 2^{-m}+\sqrt{\frac{4\ln T+2\ln(2/\delta)}{v_m}},\forall x\in u\in U_m,\forall1\le m\le m^*\right\},$$
    where $\hat{\mu}_m(u)$ is calculated when $u$ is removed from $U_m^+$. Then, it holds that $\Pr[\mathcal{E}']\ge1-\delta$.
    \label{lem:clean-prob}
\end{lemma}

\section{Proof of Theorem \ref{thm:adversarial}}
\label{appendix:proof-adversarial}

For reader's convenience, we recall several main steps in the algorithm. First, the counter $\tau$ refers to how many reward signals are received. During each round, the agent selects a random node $u$ according to the distribution $$\pi_t(u)=(1-\gamma_t)p_t+\gamma_t/|A_t|,$$
where $p_t$ is the distribution with probability mass proportional to the weight $w_\tau$. When receive a feedback from round $s$, for all $u\in A_\tau$, update $$\begin{aligned}
    w_{\tau+1}(u)&\gets w_{\tau}(u)\cdot\exp\left(\eta\hat{\mu}_{s}(\textbf{act}_{s}(u))\right)\\
    &=w_\tau(u)\cdot\exp\left(\eta\left(\frac{\mu(x_s)1\{u_s=\textbf{act}_s(u)\}+(1+4\ln T)\beta}{\pi_s(\textbf{act}_s(u))}\right)\right).
\end{aligned}$$
In addition, after each round, we calculate $$\textnormal{conf}_t(u)=\frac{1}{\beta}+\sum_{s=1}^t\frac{\beta}{\pi_s(\textbf{act}_s(u))},$$
and zoom in a node $u$ if and only if $\textnormal{conf}_t(u)<t\cdot L(u)$.

Moreover, throughout this section, we assume that $t+d_t$ is at most $T$. Clearly, this assumption does not affect the performance of the algorithm.

\subsection{Main Proof}

First, we use the following lemma to control the "estimated" regret.

\begin{lemma}
    Assume the parameter satisfies $\eta\le\beta\le 1$, $\gamma_t=(2+4\ln T)\beta|A_t|\le\frac{1}{2}$, then for any fixed active node $u^*$ in the end,
    $$\begin{aligned}
        \sum_{s=1}^T\hat{\mu}_s(\textnormal{\textbf{act}}_s(u^*))-\sum_{s=1}^T\mu(x_s)&\le-\frac{d\ln L(u^*)}{\eta}+4\sum_{s=1}^T\gamma_s+4(1+2\ln T)\beta\sum_{s=1}^T\sum_{u\in A_s}\hat{\mu}_s(u)\\
        &\quad+2\sum_{s=1}^T\langle\max\{0,p_{\tau_1(s)}-p_{\tau_0(s)}\},\hat{\mu}_s\rangle.
    \end{aligned}$$
    Here, the "inner-product" term is defined by $$\begin{aligned}
        &\langle\max\{0,p_{\tau_1(s)}-p_{\tau_0(s)}\},\hat{\mu}_s\rangle\\
        :=&\sum_{u\in A_{\tau_1(s)}}\max\left\{p_{\tau_1(s)}(u)-p_{\tau_0(s)}(\textbf{act}_s(u))\frac{L(u)^d}{L(\textbf{act}_s(u))^d},0\right\}\hat{\mu}_s(\textbf{act}_s(u)).
    \end{aligned}$$
    \label{lem:regret-bound}
\end{lemma}

\begin{proof}
    Consider the potential function $$\Phi(\tau)=\ln\left(\sum_{u\in A_\tau}w_\tau(u)\right)=\ln\left(\sum_{u\in A_\tau}L(u)^d\cdot\exp\left(\eta\sum_{\iota=1}^\tau\hat{\mu}_\tau(\textbf{act}_\iota(u))\right)\right),$$
    where for notational simplicity, we use $A_\tau=A_t$ if the $\tau$-th observation is observed after round $t$.
    
    First, for any node $u^*$ that is active after the $\tau$-th used observation
    $$\begin{aligned}
        \Phi(\tau)-\Phi(0)&=\Phi(\tau)-0\\
        &=\ln\left(\sum_{u\in A_\tau}w_{\tau}(u)\right)\\
        &\ge\ln\left(w_{\tau}(u^*)\right)\\
        &=d\ln L(u^*)+\eta\sum_{s=1}^T\hat{\mu}_{s}(\textbf{act}_s(u^*)).
    \end{aligned}$$

    On the other hand, we want to upper-bound the potential function. We have two steps to do during the process: updating weight using EXP3 rule, and zoom-in some of the nodes using the zoom-in rule. However, zooming-in the nodes does not change the potential function. So, 
    $$\begin{aligned}
        \Phi(\tau)-\Phi(\tau-1)
        &=\ln\left(\frac{\sum_{u\in A_{\tau-1}}w_{\tau-1}(u)\exp\left(\eta\hat{\mu}_{s}(\textbf{act}_s(u))\right)}{\sum_{u\in A_{\tau-1}}w_{\tau-1}(u)}\right)\\
        &=\ln\left(\sum_{u\in A_{\tau-1}}p_{\tau-1}(u)\exp\left(\eta\hat{\mu}_{s}(\textbf{act}_s(u))\right)\right).
    \end{aligned}$$
    Recall that $$\eta\hat{\mu}_{s}(\textbf{act}_s(u))\le\frac{1+(1+4\ln T)\beta}{\pi_s(\textbf{act}_s(u))}\le\frac{|A_s|(1+(1+4\ln T)\beta)\eta}{\gamma_s}\le 1.$$
    So by the inequality $e^x\le 1+x+x^2$ for $x\le 1$ and the inequality $\ln(1+x)\le x$,
    $$\begin{aligned}
        \Phi(\tau)-\Phi(\tau-1)&\le\ln\left(\sum_{u\in A_{\tau-1}}p_{\tau-1}(u)\left(1+\eta\hat{\mu}_{s}(\textbf{act}_s(u))+(\eta\hat{\mu}_{s}(\textbf{act}_s(u)))^2\right)\right)\\
        &=\ln\left(1+\eta\sum_{u\in A_{\tau-1}}p_{\tau-1}(u)\hat{\mu}_{s}(\textbf{act}_s(u))+\eta^2\sum_{u\in A_{\tau-1}}p_{\tau-1}(u)\hat{\mu}_{s}^2(\textbf{act}_s(u))\right)\\
        &\le\eta\sum_{u\in A_{\tau-1}}p_{\tau-1}(u)\hat{\mu}_{s}(\textbf{act}_s(u))+\eta^2\sum_{u\in A_{\tau-1}}p_{\tau-1}(u)\hat{\mu}_{s}^2(\textbf{act}_s(u))\\
        &\le \Bigg[\sum_{u\in A_{\tau-1}}p_{\tau-1}(u)(\eta\hat{\mu}_{s}(\textbf{act}_s(u))+\eta^2\hat{\mu}^2_{s}(\textbf{act}_s(u)))\\
        &\qquad-\sum_{u\in A_{s}}p_{\tau_0(s)}(u)(\eta\hat{\mu}_{s}(u)+\eta^2\hat{\mu}^2_{s}(u))\Bigg]\\
        &\quad+\sum_{u\in A_{s}}p_{\tau_0(s)}(u)(\eta\hat{\mu}_{s}(u)+\eta^2\hat{\mu}^2_{s}(u))\\
        &\le\underbrace{2\eta\sum_{u\in A_{\tau-1}}\left(p_{\tau-1}(u)-p_{\tau_0(s)}(\textbf{act}_s(u))\frac{L(u)^d}{L(\textbf{act}_s(u))^d}\right)\hat{\mu}_s(\textbf{act}_s(u))}_{\text{(a)}}\\
        &\quad+\underbrace{\sum_{u\in A_{s}}p_{\tau_0(s)}(u)(\eta\hat{\mu}_{s}(u)+\eta^2\hat{\mu}^2_{s}(u))}_{\text{(b)}}.
    \end{aligned}$$
    
    We now analyze (a) and (b). For part (a), we view it as an inner-product. Specifically, the first part compares the probability at timestep $\tau-1$ and the probability at timestep $\tau(s)$, when we zoom-in both the active set into $A_{\tau-1}$ for comparison. Note that the zoom-in step does not change (a). We write $\tau_0(s)$ be the value of $\tau$ before computing $\pi_s$, and $\tau_1(s)$ be the value of $\tau$ immediately before using the feedback $g_s(x_s)$. Then, we formally write $$\text{(a)}\le2\eta\langle\max\{0,p_{\tau_1(s)}-p_{\tau_0(s)}\},\hat{\mu}_s\rangle.$$
    
    Part (b) represents what would the potential change if there is no delay. We now analyze part (b). By that $\pi_s(u)=(1-\gamma_s)p_{\tau_0(s)}(u)+\gamma_s/|A_s|$ and the assumption $\gamma_s\le\frac{1}{2}$ for all $s$, $p_{\tau_0(s)}(u)\le\frac{1}{1-\gamma_s}\pi_s(u)\le(1+2\gamma_s)\pi_s(u)$, so
    $$\begin{aligned}
        \text{(b)}&\le(1+2\gamma_s)\sum_{u\in A_{s}}\pi_{s}(u)(\eta\hat{\mu}_{s}(u)+\eta^2\hat{\mu}_{s}(u))\\
        &\le(1+2\gamma_s)\eta\sum_{u\in A_s}\pi_s(u)\frac{\mu(x_s)1\{u_s=u\}+(1+4\ln T)\beta}{\pi_s(u)}\\
        &\quad+(1+2\gamma_s)\eta^2\sum_{u\in A_s}\left(\pi_s(u)\frac{\mu(x_s)1\{u_s=u\}+(1+4\ln T)\beta}{\pi_s(u)}\right)\hat{\mu}_s(u)\\
        &\le (1+2\gamma_s)\eta\left(\mu(x_s)+(1+4\ln T)\beta|A_s|\right)\\
        &\quad+2\eta^2(1+(1+4\ln T)\beta)\sum_{u\in A_s}\hat{\mu}_s(u)\\
        &\le\eta\mu(x_s)+2\gamma_s\eta\mu(x_s)+(1+2\gamma_s)\eta\gamma_s+\eta\beta\cdot4(1+2\ln T)\sum_{u\in A_s}\hat{\mu}_s(u).
    \end{aligned}$$
    Finally, summing over $\tau$, we have $$\begin{aligned}
        \Phi(T)-\Phi(0)&\le2\eta\langle\max\{0,p_{\tau_1(s)}-p_{\tau_0(s)}\},\hat{\mu}_s\rangle\\
        &+\eta\sum_{s=1}^T\mu(x_s)+4\eta\sum_{s=1}^T\gamma_s+4(1+2\ln T)\eta\beta\sum_{s=1}^T\sum_{u\in A_s}\hat{\mu}_s(u).
    \end{aligned}$$
    Applying both upper and lower bounds of the potential, and dividing both sides by $\eta$, we get
    $$\begin{aligned}
        \sum_{s=1}^T\hat{\mu}_s(\textbf{act}_s(u^*))-\sum_{s=1}^T\mu(x_s)&\le-\frac{d\ln L(u^*)}{\eta}+4\sum_{s=1}^T\gamma_s+4(1+2\ln T)\beta\sum_{s=1}^T\sum_{u\in A_s}\hat{\mu}_s(u)\\
        &\quad+2\sum_{s=1}^T\langle\max\{0,p_{\tau_1(s)}-p_{\tau_0(s)}\},\hat{\mu}_s\rangle.
    \end{aligned}$$
\end{proof}

Using the above lemma, we can partition the "estimated" regret into several terms. The next lemma then apply concentration results to provide the true regret bound.

\begin{lemma}
    Choose parameters so that $\eta=\beta\ge\frac{1}{T}$, $\gamma_t=(2+4\ln T)\beta|A_t|$. Then, with probability $1-O(T^{-2})$, the regret is bounded by $$R(T)\le \frac{2dh(u^*)}{\eta}+48\eta\ln T\sum_{s=1}^T|A_s|+24\ln^2 T+2\sum_{s=1}^T\langle\max\{0,p_{\tau_1(s)}-p_{\tau_0(s)}\},\hat{\mu}_s\rangle.$$
    \label{lem:regret-bound-highprob}
\end{lemma}

\begin{proof}
    By Lemma \ref{lem:regret-bound}, \ref{lem:concentration-optimal} and \ref{lem:concentration-total}, with probability $1-O(T^{-2})$,
    $$\begin{aligned}
        \sum_{s=1}^T\mu_s(x^*)-\sum_{s=1}^T\mu_s(x_s)&\le\sum_{s=1}^T\mu_s(\textbf{act}_s(u))-\sum_{s=1}^T\mu_s(x_s)+\frac{6\ln T}{\beta}\\
        &\le-\frac{d\ln L(u^*)}{\eta}+4\sum_{s=1}^T\gamma_s+4(1+2\ln T)\eta\sum_{s=1}^T\sum_{u\in A_s}\hat{\mu}_s(u)\\
        &\qquad+2\sum_{s=1}^T\langle\max\{0,p_{\tau_1(s)}-p_{\tau_0(s)}\},\hat{\mu}_s\rangle.\\
        (\text{Use Lemma \ref{lem:concentration-total}})&\le\frac{2dh(u^*)}{\eta}+4\sum_{s=1}^T\gamma_s\\
        &\quad+4(1+2\ln T)\left[\beta\sum_{s=1}^T\sum_{u\in A_s}\mu_s(u)+2\ln T+\beta\sum_{s=1}^T|A_s|\right]\\
        &\quad+2\sum_{s=1}^T\langle\max\{0,p_{\tau_1(s)}-p_{\tau_0(s)}\},\hat{\mu}_s\rangle.\\
        (\text{Use }\gamma_s=(2+4\ln T)\eta)&=\frac{2dh(u^*)}{\eta}+16(1+2\ln T)\eta\sum_{s=1}^T|A_s|+8(1+2\ln T)\ln T\\
        &\quad+2\sum_{s=1}^T\langle\max\{0,p_{\tau_1(s)}-p_{\tau_0(s)}\},\hat{\mu}_s\rangle.\\
    \end{aligned}$$
    Then we obtain the result.
\end{proof}

\begin{lemma}
    $$\mathbb{E}\left[\sum_{s=1}^T\Vert\max\left\{p_{\tau_1(s)}-p_{\tau_0(s)},0\right\}\Vert_1\right]\le8\eta\sum_{s=1}^Td_s.$$
    \label{lem:delay-expectation}
\end{lemma}

\begin{proof}
    $$\Vert\max\left\{p_{\tau_1(s)}-p_{\tau_0(s)},0\right\}\Vert_1\le\sum_{\tau=\tau_0(s)+1}^{\tau_1(s)}\Vert\max\{p_{\tau}-p_{\tau-1},0\}\Vert_1.$$
    Now we focus on two adjacent steps, and note that zooming-in does not change the 1-norm. Assume that the $\tau$-th observed reward comes from the $t$-th pull. Then, for each $u\in A_{\tau-1}$, $$\begin{aligned}
        p_\tau(u)-p_{\tau-1}(u)&=\frac{w_{\tau-1}(u)\cdot\exp\left(\eta\hat{\mu}_t(\textbf{act}_t(u))\right)}{\sum_{v\in A_{\tau-1}}w_{\tau-1}(v)\cdot\exp\left(\eta\hat{\mu}_t(\textbf{act}_t(v))\right)}-\frac{w_{\tau-1}(u)}{\sum_{v\in A_{\tau-1}}w_{\tau-1}(v)}\\
        &\le\frac{w_{\tau-1}(u)\cdot\exp\left(\eta\hat{\mu}_t(\textbf{act}_t(u))\right)}{\sum_{v\in A_{\tau-1}}w_{\tau-1}(v)}-\frac{w_{\tau-1}(u)}{\sum_{v\in A_{\tau-1}}w_{\tau-1}(v)}\\
        &\le p_{\tau-1}(u)\cdot\left[\exp\left(\eta\hat{\mu}_t(\textbf{act}_t(u))\right)-1\right].
    \end{aligned}$$
    Note that the definition of the parameters ensures that $\eta\hat{\mu}_t(\textbf{act}_t(u))\le 1$, so using $e^x\le 1+2x$ for $x\le 1$,
    $$\begin{aligned}
        p_\tau(u)-p_{\tau-1}(u)&\le 2p_{\tau-1}(u)\cdot \eta\left(\frac{g(x_t)1\{u_t=\textbf{act}_t(u)\}}{\pi_t(\textbf{act}_t(u))}+\frac{(1+4\ln T)\beta}{\pi_t(\textbf{act}_t(u))}\right)\\
        &\le 2p_{\tau-1}(u)\cdot\eta\left(\frac{1\{u_t=\textbf{act}_t(u)\}}{\pi_t(\textbf{act}_t(u))}+\frac{(1+4\ln T)\beta|A_t|}{\gamma_t}\right)\\
        &\le 2p_{\tau-1}(u)\cdot\eta\left(\frac{1\{u_t=\textbf{act}_t(u)\}}{\pi_t(\textbf{act}_t(u))}+1\right).
    \end{aligned}$$
    Note that $u_t$ is independent of $\{p_0,p_1,\dots,p_{\tau-1}\}$, since its feedback is never used until the $\tau$-th update. Therefore, let $\mathcal{F}_\tau$ be the information filtration used for computing $\{p_0,\dots,p_{\tau-1}\}$, we have $$\begin{aligned}
        \mathbb{E}[\max\{p_\tau(u)-p_{\tau-1}(u),0\}|\mathcal{F}_{\tau-1}]&\le2p_{\tau-1}\eta\left(\frac{1}{\pi_t(\textbf{act}_t(u))}\mathbb{E}\left[1\{u_t=\textbf{act}_t(u)|\mathcal{F}_{\tau-1}\}\right]+1\right)\\
        &=2p_{\tau-1}\eta\left(\frac{1}{\pi_t(\textbf{act}_t(u))}\mathbb{E}\left[1\{u_t=\textbf{act}_t(u)|\pi_t\}\right]+1\right)\\
        &=2p_{\tau-1}\eta(1+1)=4p_{\tau-1}\eta.
    \end{aligned}$$
    Summing over $u$, we get $$\mathbb{E}\left[\Vert\max\{p_{\tau}-p_{\tau-1},0\}\Vert_1|\mathcal{F}_{\tau-1}\right]\le4\eta\sum_{u\in A_{\tau-1}}p_{\tau-1}(u)\le 4\eta.$$
    This implies that $$\mathbb{E}\left[\Vert\max\left\{p_{\tau_1(s)}-p_{\tau_0(s)},0\right\}\Vert_1\right]\le 4\eta(\tau_1(s)-\tau_0(s)).$$
    
    We now analyze $$\sum_{t=1}^T[\tau_1(t)-\tau_0(t)],$$
    which is determined by the delays $d_t$. We apply similar approach as Lemma 4 of \citet{bistritz2019onlineexp3}. Recall that $O_t=\{s:s+d_s=t\}$ is the set of feedback received after $t$-th pull. Without loss of generality, we assume that if multiple feedback arrives at the same time, we perform the update in the order of their source timesteps (i.e., $s$). It is clear that changing the order does not affect the sum we are considering. 

    For $\tau_1(t)-\tau_0(t)$, there are two kinds of observations contributes to the difference. $s>t$ with $s+d_s<t+d_t$, i.e., $s>t$ but $s$ is observed before observing $t$; and $s<t$ with $t\le s+d_s\le t+d_t$, i.e., $s<t$, not observed before $t$-th pull, but observed before observing $t$.
    
    For the first part, it is clear that $s\in(t,t+d_t)$, so it is directly bounded by $d_t$. For the second part, we fix such $s$ and consider how many $t$ would satisfy such condition. It is then clear that $t\in(s,s+d_s]$, so for each $s$, it contributes to at most $d_s$ different $t$'s. Therefore, $$\sum_{t=1}^T[\tau_1(t)-\tau_0(t)]\le2\sum_{t=1}^Td_t.$$
    The lemma is then obtained by combining the results.
\end{proof}

\begin{lemma}
    $$\mathbb{E}\left[\sum_{s=1}^T\langle\max\{0,p_{\tau_1(s)}-p_{\tau_0(s)}\},\hat{\mu}_s\rangle\right]\le 16\eta\sum_{s=1}^Td_s.$$
    \label{lem:delay-inner-product-expectation}
\end{lemma}

\begin{proof}
    Note that given $\pi_s$, $p_{\tau_1(s)}$ is independent with $x_s$ and $\hat{\mu}_s(x_s)$, since the value of $x_s$ and $\hat{\mu}(x_s)$ is never used before updating the weight $w_{\tau_1(s)+1}$. Therefore, $$\begin{aligned}
        \mathbb{E}\left[\langle\max\{0,p_{\tau_1(s)}-p_{\tau_0(s)}\},\hat{\mu}_s\rangle|\pi_s\right]&=\langle\mathbb{E}\left[\max\{0,p_{\tau_1(s)}-p_{\tau_0(s)}\}|\pi_s\right],\mathbb{E}\left[\hat{\mu}_s|\pi_s\right]\rangle\\
        &\le\Vert\mathbb{E}\left[\max\{0,p_{\tau_1(s)}-p_{\tau_0(s)}\}|\pi_s\right]\Vert_1\cdot\Vert\mathbb{E}\left[\hat{\mu}_s|\pi_s\right]\Vert_\infty.
    \end{aligned}$$
    However, since for any $u\in A_s$, $\hat{\mu}_s(u)\ge0$, and $$\begin{aligned}
        \mathbb{E}[\hat{\mu}_s(u)|\pi_s]&=\frac{1}{\pi_s(u)}\mathbb{E}[\mu(x_s)1\{u=u_s\}]+\frac{(1+4\ln T)\beta}{\pi_s(u)}\\
        &\le\frac{\Pr\{u=u_s\}}{\pi_s(u)}+\frac{|A_s|(1+4\ln T)\beta}{\gamma_s}\\
        &\le1+1=2,
    \end{aligned}$$
    we have $\Vert\mathbb{E}\left[\hat{\mu}_s|\pi_s\right]\Vert_\infty\le 2$. Therefore, taking sum over $s$ and apply Lemma \ref{lem:delay-expectation}, we have
    $$\mathbb{E}\left[\sum_{s=1}^T\langle\max\{0,p_{\tau_1(s)}-p_{\tau_0(s)}\},\hat{\mu}_s\rangle\right]\le2\mathbb{E}\left[\sum_{s=1}^T\Vert\max\left\{p_{\tau_1(s)}-p_{\tau_0(s)},0\right\}\Vert_1\right]\le 16\eta\sum_{s=1}^Td_s.$$
\end{proof}

We now provide the proof of theorem.

\begin{proof}[Proof of Theorem \ref{thm:adversarial}]
    Choose $$\eta=\beta=\Theta\left(\frac{\sqrt{d\ln T}}{\sqrt{2^d}T^{(d_z+1)/(d_z+2)}c^{\frac{1}{d_z+2}}\sqrt{\ln T}+\sqrt{D}}\right),$$
    and
    $$\gamma_t=(2+4\ln T)\beta|A_t|.$$
    By Lemma \ref{lem:tree-size}, with probability $1-O(T^{-2})$, $|A_T|=O\left(2^dT^{\frac{d_z}{d_z+2}}c^{\frac{2}{d_z+2}}\right)$, so $\gamma_t\le\frac{1}{2}$ can be satisfied. by Lemmas \ref{lem:regret-bound-highprob} and \ref{lem:tree-height}, with probability $1-O(T^{-2})$,
    $$\begin{aligned}
        R(T)&\le\frac{2d\ln T}{\eta}+48\eta\ln T\cdot T|A_T|+24\ln^2 T+2\sum_{s=1}^T\langle\max\{0,p_{\tau_1(s)}-p_{\tau_0(s)}\},\hat{\mu}_s\rangle.
    \end{aligned}$$
    Apply the bound by \ref{lem:tree-size}, 
    $$\begin{aligned}
        R(T)&\le O\left(\sqrt{d2^d}T^{\frac{d_z+1}{d_z+2}}c^{\frac{1}{d_z+2}}\ln T\right)+2\sum_{s=1}^T\langle\max\{0,p_{\tau_1(s)}-p_{\tau_0(s)}\},\hat{\mu}_s\rangle.
    \end{aligned}$$
    Taking expectation and apply Lemma \ref{lem:delay-inner-product-expectation}, since the total regret is always upper-bounded by $T$,
    $$\begin{aligned}
        \mathbb{E}[R(T)]&\le O\left(\sqrt{d2^d}T^{\frac{d_z+1}{d_z+2}}c^{\frac{1}{d_z+2}}\ln T+\sqrt{dD\ln T}\right)+2\sum_{s=1}^T\langle\max\{0,p_{\tau_1(s)}-p_{\tau_0(s)}\},\hat{\mu}_s\rangle\\
        &\le O\left(\sqrt{d2^d}T^{\frac{d_z+1}{d_z+2}}c^{\frac{1}{d_z+2}}\ln T+\sqrt{dD\ln T}\right)+O\left(\eta D\right)\\
        &\le O\left(\sqrt{d2^d}T^{\frac{d_z+1}{d_z+2}}c^{\frac{1}{d_z+2}}\ln T+\sqrt{dD\ln T}\right).
    \end{aligned}$$
\end{proof}

\subsection{Lemmas for Zooming Processes}

The following lemmas are similar to those by \citet{podimata2021adaptive}, where they are rewritten to fit our algorithm.

\begin{lemma}
    Suppose a non-root node $u$ is deactivated at round $t_{d}(u)$. Let $v$ be its parent, then $t_d(u)\ge2t_d(v)-1$.
    \label{lem:deactivate-scale}
\end{lemma}

\begin{proof}
    The node $v$ does not satisfy the zoom-in rule at round $t_d(v)-1$, so $$\sum_{t=1}^{t_d(v)-1}\frac{\beta}{\pi_t(\textbf{act}_t(v))}+\frac{1}{\beta}\ge(t_d(v)-1)L(v)=(2t_d(v)-2)L(u).$$
    Note that $\textnormal{conf}_t(u)$ is non-decreasing in $t$, we know that $\textnormal{conf}_t(u)\ge(2t_d(v)-2)L(u)$,
    meaning that $t_d(u)\ge 2t_d(v)-1$.
\end{proof}

\begin{lemma}
    The height of any active node is at most $\log_2\frac{1}{\beta}$. Specifically, as long as $\frac{1}{\beta}=O(T)$, the height is bounded by $O(\log T)$.
    \label{lem:tree-height}
\end{lemma}

\begin{proof}
    When a node $u$ zooms-in at round $t$, it does not satisfy the zoom-in rule at round $t-1$, so $$\textnormal{conf}_t(u)<tL(u),\quad\textnormal{conf}_{t-1}(u)\ge(t-1)L(u).$$
    Taking the difference, $$\frac{\beta}{\pi_t(u)}<L(u)=2^{-h(u)},$$
    which means that $$h(u)\le\log_2\left(\frac{\pi_t(u)}{\beta}\right)\le\log_2\frac{1}{\beta}.$$
\end{proof}

\begin{lemma}
    Suppose $\beta\ge1/T$. For a node $u$ active at round $t$, $$\sum_{s=1}^tL(\textnormal{\textbf{act}}_s(u))\le 2tL(u)\log_2 T.$$
    \label{lem:total-diameter}
\end{lemma}

\begin{proof}
    Let $v_0,v_1,\dots,v_{h(u)}$ be the path from root to $u$, where $v_{h(u)}=u$. Then,
    $$\begin{aligned}
        \sum_{s=1}^tL(\textbf{act}_s(u))&=t_d(v_0)L(v_0)+\sum_{s=1}^{h(u)-1}[t_d(v_s)-t_d(v_{s-1})]L(v_s)+(t-t_d(v_{h(u)-1}))L(u)\\
        &=\sum_{s=0}^{h(u)-1}t_d(v_s)[L(v_s)-L(v_{s+1})]+tL(u).
    \end{aligned}$$
    By Lemma \ref{lem:deactivate-scale}, $$t_d(v_s)\le\frac{t_d(v_{s+1})+1}{2}\le\cdots\le 2^{s-h(s)+1}t_d(v_{h(s)-1})+\sum_{r=1}^{s-h(s)+1}2^{-r}\le2^{s-h(s)+1}t_d(v_{h(s)-1})+1.$$
    So $$\begin{aligned}
        \sum_{s=1}^tL(\textbf{act}_s(u))&\le\frac{1}{2}\sum_{s=0}^{h(u)-1}L(v_s)(2^{s-h(s)+1}t_d(v_{h(s)-1})+1)+tL(u)\\
        &=\frac{1}{2}\sum_{s=0}^{h(u)-1}2^{h(s)-s}L(u)(2^{s-h(s)+1}t_d(v_{h(s)-1})+1)+tL(u)\\
        &\le L(u)+tL(u)+\sum_{s=0}^{h(u)-1}L(u)t_d(v_{h(s)-1})\\
        &\le L(u)+tL(u)+tL(u)h(u)\le2tL(u)\log_2 T
    \end{aligned}$$
\end{proof}

\begin{lemma} (Lemma C.4 by \citet{podimata2021adaptive})
    Let $[t_a(u),t_d(u)]$ be the interval where $u$ is active, and assume that $t_d(u)<T$, i.e., $u$ satisfies the zoom-in rule in the end. Then,
    $$\sum_{t=t_a(u)}^{t_d(u)}\pi_t(u)\ge\frac{1}{9L(u)^2}.$$
    \label{lem:total-mass}
\end{lemma}

\begin{proof}
    Let $v$ be the parent of $u$. $$\begin{aligned}
        t_d(u)L(u)&>\sum_{t=1}^{t_d(u)}\frac{\beta}{\pi_t(\textnormal{\textbf{act}}_t(u))}+\frac{1}{\beta}\\
        &\ge\sum_{t=t_a(u)}^{t_d(u)}\frac{\beta}{\pi_t(u)}.
    \end{aligned}$$
    Since the function $x\mapsto x^{-1}$ is convex, Jensen's inequality shows that $\frac{1}{n}\sum_{i=1}^nx_i^{-1}\ge\frac{n}{\sum_{i=1}^nx_i}$, so $$\sum_{t=t_a(u)}^{t_d(u)}\frac{\beta}{\pi_t(u)}\ge\frac{\beta(t_d(u)-t_d(v))^2}{\sum_{t=t_a(u)}^{t_d(u)}\pi_t(u)}\ge\frac{\beta(t_d(u)-(t_d(u)+1)/2)^2}{\sum_{t=t_a(u)}^{t_d(u)}\pi_t(u)}\ge\frac{\beta t_d(u)^2}{9\sum_{t=t_a(u)}^{t_d(u)}\pi_t(u)},$$
    where the second last inequality comes from Lemma \ref{lem:deactivate-scale}. Combining the inequalities,
    $$\sum_{t=t_a(u)}^{t_d(u)}\pi_t(u)\ge\frac{\beta t_d(u)^2}{9t_d(u)L(u)}\ge\frac{\beta t_d(u)}{9L(u)}.$$
    Since $\textnormal{conf}_t(u)\ge\frac{1}{\beta}$, the zoom-in rule ensures that $t_d(u)L(u)>\frac{1}{\beta}$, so
    $$\sum_{t=t_a(u)}^{t_d(u)}\pi_t(u)\ge\frac{1}{9L(u)^2}.$$
\end{proof}

\subsection{Lemmas for Concentration}

The first lemma comes from the results by \citet{podimata2021adaptive}, which only depends on the probability $\pi_t$ used when we are sampling arms. Denote the average return for node $u$ as $\mu_t(u)$.

\begin{lemma} (Lemma C.9 of \citet{podimata2021adaptive})
    Fix a round $t$ and a sequence of sets $A_s'\subseteq A_s$. With probability at least $1-2\delta$,
    $$\left|\beta\sum_{s=1}^t\sum_{u\in A_s'}\mu_s(u)-\beta\sum_{s=1}^t\sum_{u\in A_s'}\frac{1\{u=u_s\}\mu_s(x_s)}{\pi_s(u)}\right|\le\beta^2\sum_{s=1}^t\sum_{u\in A_s'}\frac{1}{\pi_s(u)}+\ln(1/\delta).$$
    \label{lem:c9}
\end{lemma}

\begin{lemma} (Corollary C.1 of \citet{podimata2021adaptive})
    Fix a round $t$ and a node $u\in A_t$. With probability at least $1-2\delta$,
    $$\left|\sum_{s=1}^t\mu_s(\textnormal{\textbf{act}}_s(u))-\sum_{s=1}^t\frac{\mu_s(x_s)1\{u_s=\textnormal{\textbf{act}}_s(u)\}}{\pi_s(\textnormal{\textbf{act}}_s(u))}\right|\le\beta\sum_{s=1}^t\frac{1}{\pi_s(\textnormal{\textbf{act}}_s(u))}+\frac{\ln(1/\delta)}{\beta}.$$
    \label{lem:c9-single}
\end{lemma}

\begin{proof}
    Apply Lemma \ref{lem:c9} to singleton $A_s'=\{\textbf{act}_s(u)\}$
\end{proof}

\begin{lemma}
    Fix a round $t$ and a node $u\in A_{t+1}$. Let $\tilde{O}_t=\{s:s+d_s\le t\}$. Then with probability at least $1-2\delta$,
    $$\left|\sum_{s\in\hat{O}_t}\mu_s(\textnormal{\textbf{act}}_s(u))-\sum_{s\in\hat{O}_t}\frac{\mu_s(x_s)1\{u_s=\textnormal{\textbf{act}}_s(u)\}}{\pi_s(\textnormal{\textbf{act}}_s(u))}\right|\le\beta\sum_{s\in\hat{O}_t}\frac{1}{\pi_s(\textnormal{\textbf{act}}_s(u))}+\frac{\ln(1/\delta)}{\beta}.$$
    \label{lem:c9-single2}
\end{lemma}

\begin{proof}
    For $s\in\tilde{O}_t$, take $A_s'=\{\textbf{act}_s(u)\}$. Otherwise take $A_s'=\varnothing$. Then apply Lemma \ref{lem:c9}.
\end{proof}

The next few lemmas follows \citet{podimata2021adaptive}, but we write the proof explicitly to refine some constants.

\begin{lemma} 
    Fix a round $t$, then with probability at least $1-2\delta$, $$\beta\sum_{s=1}^t\sum_{u\in A_s}\hat{\mu}_s(u)-\beta\sum_{s=1}^t\sum_{u\in A_s}\mu_s(u)\le \ln(1/\delta)+\beta\sum_{s=1}^T|A_s|.$$
    \label{lem:concentration-total}
\end{lemma}

\begin{proof}
    By Lemma \ref{lem:c9}, with probability at least $1-2\delta$,
    $$\begin{aligned}
        &\quad\beta\sum_{s=1}^t\sum_{u\in A_s}\hat{\mu}_s(u)-\beta\sum_{s=1}^t\sum_{u\in A_s}\mu_s(u)\\
        &\le\beta\sum_{s=1}^t\sum_{u\in A_s}\left[\frac{\mu_s(x_s)1\{u_t=u\}}{\pi_s(u)}-\mu_s(u)+\frac{(1+4\ln T)\beta}{\pi_s(u)}\right]\\
        &\le\ln(1/\delta)+\beta\sum_{s=1}^t\sum_{u\in A_s}\frac{\beta+(1+4\ln T)\beta}{\pi_s(u)}\\
        &\le\ln(1/\delta)+\beta\sum_{s=1}^t\sum_{u\in A_s}\frac{|A_s|(2+4\ln T)\beta}{\gamma_s}\\
        &\le\ln(1/\delta)+\beta\sum_{s=1}^t|A_s|.
    \end{aligned}$$
\end{proof}

\begin{lemma}
    Fix a round $t$ and active node $u$. Let $x^*\in u$ be an arbitrary arm. With probability at least $1-2\delta$,
    $$\sum_{s=1}^t\hat{\mu}_s(\textnormal{\textbf{act}}_s(u))\ge\sum_{s=1}^t\mu_s(x^*)-\frac{4\ln T+\ln(1/\delta)}{\beta}.$$
    \label{lem:concentration-optimal}
\end{lemma}

\begin{proof}
    By Lemma \ref{lem:c9-single}, with probability $1-2\delta$,
    $$\begin{aligned}
        \sum_{s=1}^t\hat{\mu}_s(u)&=\sum_{s=1}^t\frac{\mu_s(x_s)1\{u_s=u\}}{\pi_s(u)}+\sum_{s=1}^t\frac{(1+4\ln T)\beta}{\pi_s(u)}\\
        &\ge\sum_{s=1}^t\mu_s(\textbf{act}_s(u))+\sum_{s=1}^t\frac{4\beta\ln T}{\pi_s(u)}-\frac{\ln(1/\delta)}{\beta}\\
        &\ge\sum_{s=1}^t\mu_s(x^*)-\sum_{s=1}^tL(\textbf{act}_s(u))+\sum_{s=1}^t\frac{4\beta\ln T}{\pi_s(u)}-\frac{\ln(1/\delta)}{\beta}.
    \end{aligned}$$
    By the zoom-in rule and Lemma \ref{lem:total-diameter}, $$\sum_{s=1}^t\frac{4\beta\ln T}{\pi_s(\textbf{act}_s(u))}-\sum_{s=1}^tL(\textbf{act}_s(u))\ge4(t-1)L(u)\ln T-\frac{4\ln T}{\beta}-2tL(u)\log_2T\ge-\frac{4\ln T}{\beta}.$$
\end{proof}

\begin{lemma}
    Fix a round $t$ and active node $u\in A_{t+1}$. Let $x^*\in u$ be an arbitrary arm. With probability at least $1-2\delta$,
    $$\sum_{s\in\tilde{O}_t}\hat{\mu}_s(\textnormal{\textbf{act}}_s(u))\ge\sum_{s\in\tilde{O}_t}\mu_s(x^*)-2tL(u)\log_2 T-\frac{\ln(1/\delta)}{\beta}.$$
    \label{lem:concentration-optimal2}
\end{lemma}

\begin{proof}
    Similar to the proof of Lemma \ref{lem:concentration-optimal}, but we directly use $\sum_{s\in\tilde{O}_t}\frac{4\beta\ln T}{\pi_s(u)}\ge0$ instead.
\end{proof}

\subsection{Bound on the Number of Activated Nodes}

In this section, we control the number of activated nodes. We will provide a worst-case estimation through covering dimension and a high-probability through the modified adversarial zooming dimension. Recall $\tilde{O}_t:=\{s:s+d_s\le t\}$, that is, $\tilde{O}_t$ contains those rewards observable before $t$-th pull. Then, $$w_{\tau_0(t)}(u)=L(u)^{d}\exp\left(\eta\sum_{s\in\tilde{O}_{t-1}}\hat{\mu}_s(\textbf{act}_s(u))\right).$$

\begin{lemma}
    If $u$ is zoomed in after round $t$ with $L(u)\le\frac{1}{3+4\ln T}$, then for any active node $u^*$,
    $$\sum_{s\in\tilde{O}_{t-1}}[\hat{\mu}_s(\textnormal{\textbf{act}}_s(u^*))-\hat{\mu}_s(\textnormal{\textbf{act}}_s(u))]\le\frac{\ln\left(\frac{L(u)^d}{\beta L(u^*)^d}\right)}{\eta}.$$
    \label{lem:estimated-advgap}
\end{lemma}

\begin{proof}
    Recall that when proving Lemma \ref{lem:tree-height}, we show that $\frac{\beta}{\pi_t(u)}\le L(u)$ must be satisfied. Then, by $\pi_t(u)=(1-\gamma_t)p_{\tau_0(t)}(u)+\frac{\gamma_t}{|A_t|}$, $$p_t(u)=\frac{\pi_t(u)-\gamma_t/|A_t|}{1-\gamma_t}\ge\frac{\beta}{L(u)}-(2+4\ln T)\beta\ge\beta,$$
    where the last inequality is by $L(u)^{-1}\ge3+4\ln T$.

    By the definition of $p_t(u)$, this implies that
    $$\begin{aligned}
        w_{\tau_0(t)}(u)&=L(u)^{d}\exp\left(\eta\sum_{s\in\tilde{O}_{t-1}}\hat{\mu}_s(\textnormal{\textbf{act}}_s(u))\right)\\
        &\ge\beta\sum_{v\in A_t}w_{\tau_0(t)}(v)\\
        &\ge\beta L(u^*)^d\exp\left(\eta\sum_{s\in\tilde{O}_{t-1}}\hat{\mu}_s(\textbf{act}_s(u^*))\right).
    \end{aligned}$$
    Taking logarithm on both sides, $$\sum_{s\in\tilde{O}_{t-1}}[\hat{\mu}_s(\textbf{act}_s(u^*))-\hat{\mu}_s(\textbf{act}_s(u))]\le\frac{\ln\left(\frac{L(u)^d}{\beta L(u^*)^d}\right)}{\eta}.$$
\end{proof}

\begin{lemma}
    Fix a node $u$ with $L(u)\le\frac{1}{3+4\ln T}$ and suppose it is zoomed-in in round $t$. Then with probability at least $1-O(T^{-3})$,
    \label{lem:advgap-single}
    $$\frac{1}{t}\sum_{s\in\tilde{O}_{t-1}}[\mu_s(x^*)-\mu_s(x)]\le O(L(u)\ln T).$$
\end{lemma}

\begin{proof}
    Let $x\in u$, and let $x^*\in u^*$ be an arbitrary arm. By Lemma \ref{lem:c9-single2}, with probability at least $1-2\delta$,
    $$\begin{aligned}
        &\qquad\sum_{s\in\tilde{O}_{t-1}}[\hat{\mu}_s(\textbf{act}_s(u))-\mu_s(x)]\\
        &\le\sum_{s\in\tilde{O}_{t-1}}[\mu_s(\textbf{act}_s(u))-\mu_s(x)]+\sum_{s\in\tilde{O}_{t-1}}\frac{\beta+(1+4\ln T)\beta}{\pi_s(\textbf{act}_s(u))}+\frac{\ln(1/\delta)}{\beta}\\
        &\le\sum_{s=1}^tL(\textbf{act}_s(u))+\sum_{s=1}^t\frac{(2+4\ln T)\beta}{\pi_s(\textbf{act}_s(u))}+\frac{\ln(1/\delta)}{\beta}.
    \end{aligned}$$
    By Lemma \ref{lem:total-diameter}, $\sum_{s=1}^tL(\textbf{act}_s(u))\le 4tL(u)\ln T$. Also by the zoom-in rule, $$\sum_{s=1}^t\frac{\beta}{\pi_s(\textbf{act}_s(u))}\le tL(u).$$
    Hence, $$\sum_{s\in\tilde{O}_{t-1}}[\hat{\mu}_s(\textbf{act}_s(u))-\mu_s(x)]\le 10tL(u)\ln T+\frac{\ln(1/\delta)}{\beta}.$$
    Then, use Lemma \ref{lem:concentration-optimal2} and Lemma \ref{lem:estimated-advgap}, 
    $$\begin{aligned}
        \sum_{s\in\tilde{O}_{t-1}}[\mu_s(x^*)-\mu_s(x)]
        &\le\sum_{s\in\tilde{O}_{t-1}}[\mu_s(x^*)-\hat{\mu}_s(\textbf{act}_s(u^*))]\\
        &\quad+\sum_{s\in\tilde{O}_{t-1}}[\hat{\mu}_s(\textbf{act}_s(u^*)-\hat{\mu}_s(\textbf{act}_s(u))]\\
        &\quad+\sum_{s\in\tilde{O}_{t-1}}[\hat{\mu}_s(\textbf{act}_s(u))-\mu_s(x)]\\
        &\le2tL(u)\log_2T+\frac{2\ln(1/\delta)}{\beta}+\frac{\ln\left(\frac{2L(u)^d}{\beta L(u^*)^d}\right)}{\eta}+10tL(u)\ln T\\
        &\le\frac{2\ln(1/\delta)+2d\ln T+\ln(2/\eta)}{\eta}+14tL(u)\ln T.
    \end{aligned}$$
    Note that $L(u)\ge\frac{1}{\beta}$ by Lemma \ref{lem:tree-height}.  If we take $\delta=T^{-3}$ and $\eta=\beta\ge\frac{1}{T}$, we have $$\frac{1}{t}\sum_{s\in\tilde{O}_{t-1}}[\mu_s(x^*)-\mu_s(x)]\le O(L(u)\ln T).$$
\end{proof}

\begin{lemma}
    With probability $1-O(T^{-2})$, the total number of activated nodes are bounded by
    $$|A_T|\lesssim 2^dT^{\frac{d_z}{d_z+2}}c^{\frac{2}{d_z+2}},$$
    where $d_z$ is the adversarial $c$-zooming dimension.
    \label{lem:tree-size}
\end{lemma}

\begin{proof}
    We use the method by \citet{podimata2021adaptive} to prove this lemma. First, we show that in the worst case, the total number of zoomed-in nodes are at most $O(T)$. Note that for a node $u$, if we want $u$ to zoom-in, according to Lemma \ref{lem:total-mass}, we need to spend at least $\frac{1}{9L(u)^2}$ probability mass on it. Thus, in order to maximize the number of zoomed-in nodes, one should perform greedily and spend probability mass on large nodes. For a node with diameter $l$, the number of such nodes is at most $l^{-d}$. So if the layer of nodes with diameter $l$ are all zoomed-in, $l$ must satisfy $$l^{-d}\cdot l^{-2}\le 9T,$$
    which indicates that $l^{-d}=O(T^{\frac{d}{d+2}})=O(T)$. Therefore, the total number of zoomed-in nodes are at most $O(T)$.

    Now, apply union bound and Lemma \ref{lem:advgap-single}, we know that with probability $1-O(T^{-2})$, the inquality in Lemma \ref{lem:advgap-single} holds for all $u$ with $L(u)\le\frac{1}{3+4\ln T}$. For nodes with $L(u)>\frac{1}{3+4\ln T}$, there are at most $O(\log^dT)$ such nodes, which is smaller than any $\Theta(T^\alpha)$ with $\alpha>0$. Note that by Lemma \ref{lem:total-mass} again, $t_d(u)\ge\frac{1}{9L(u)^2}$, where $t_d(u)$ refers to the round when $u$ is de-activated. So by our definition of the adversarial zooming dimension, for a diameter $l$, the number of $u$ with $L(u)=l$ that zooms-in is at most $c\cdot l^{-d_z}$. Similarly to the greedy method, in order to maximize the number of zoomed-in nodes, one must perform the greedy method to fill large nodes first. Let $H$ be the height of last layer that is completed in the greedy method, then it must satisfy $$c\sum_{h=0}^H(2^{-h})^{-d_z}\cdot (2^{-h})^{-2}\le 9T,$$
    which implies that $c2^{(2+d_z)H}\le T$, so the total number of activated nodes are at most $$|A_T|\le O(c2^{d_zH})\le O(c2^d\cdot (T/c)^{\frac{d_z}{d_z+2}})=O\left(2^dT^{\frac{d_z}{d_z+2}}c^{\frac{2}{d_z+2}}\right).$$
\end{proof}

\end{document}